\documentclass[onecolumn,11pt]{article}
\usepackage[top=1.2in, bottom=1.2in, left=1.1in, right=1.1in]{geometry}
\usepackage[utf8]{inputenc}
\usepackage[T1]{fontenc}
\usepackage{listings}
\usepackage{setspace}
\usepackage{etoolbox}
\usepackage{xcolor}
\usepackage{courier}
\usepackage{caption}
\usepackage{graphicx}
\usepackage{cprotect}
\usepackage{xspace}
\usepackage[affil-it]{authblk} 
\usepackage[symbol]{footmisc}
\graphicspath{ {./figures/} }
\usepackage{mathpazo}
\patchcmd{\maketitle}{\@fnsymbol}{\@alph}{}{}  

\usepackage[numbers,sort&compress]{natbib}
\usepackage{tikz}
\usetikzlibrary{arrows,calc,positioning,shadows,shapes,shapes.geometric}
\usepackage{pgfplots}
\usepackage{adjustbox}
\usepackage{pifont}
\usepackage{booktabs}
\usepackage{array}
\usepackage{amsmath}
\usepackage{siunitx}
\usepackage{wrapfig}

\usepackage{hyperref}       
\usepackage{url}            
\usepackage{booktabs}       
\usepackage{amsfonts}       
\usepackage{nicefrac}       
\usepackage{microtype}      
\usepackage[table]{xcolor}   
\usepackage{amsthm}
\usepackage{enumitem}
\newcommand{\cmark}{\ding{51}}
\newcommand{\xmark}{\ding{55}}
\usepackage{amssymb}
\usepackage{fix-cm} 

\usepackage[nohyperlinks,nolist]{acronym}

\usepackage[title]{appendix}

\newtheorem{assumption}{Assumption}

\newtheorem{proposition}{Proposition}
\newtheorem{theorem}{Theorem}
\newtheorem{corollary}{Corollary}

\newtheorem{lemma}{Lemma}

\newcommand{\R}{\mathbb{R}}

\newcommand{\E}{\mathbb{E}}

\newcommand{\bz}{\mathbf{z}}
\newcommand{\bx}{\mathbf{x}}
\newcommand{\bc}{\mathbf{c}}
\newcommand{\bw}{\mathbf{w}}

\newcommand{\balpha}{\boldsymbol{\alpha}}

\newcommand{\bW}{\mathbf{W}}

\newcommand{\Softmax}{\operatorname{Softmax}}
\newcommand{\argmin}{\operatorname*{arg\,min}}

\newcommand{\be}{\mathbf{e}}

\title{\textbf{\huge{Learning Individual Dynamics from \\Sparse Cross-Sectional Snapshots}}\vspace{-.1in}}
\author[*,1]{Christian Lagemann}
\author[*,2]{Kai Lagemann}
\author[3]{Steven L. Brunton}
\author[4]{Sach Mukherjee\vspace{-.1in}}
{\small
\affil[1]{\footnotesize Statistics and Machine Learning, German Center for Neurodegenerative Diseases (DZNE), Bonn, Germany}
\affil[2]{\footnotesize MediaTek Research, London, United Kingdom}
\affil[3]{\footnotesize Department of Mechanical Engineering \& 
AI Institute in Dynamic Systems, University of Washington, Seattle, United States}
\affil[4]{\footnotesize DZNE \& University of Bonn, Bonn, Germany and 
University of Cambridge, Cambridge, United Kingdom}
}
\date{}

\usepackage{mathrsfs}
\usepackage{mathtools}
\usepackage{dsfont}
\usepackage{bm}
\usepackage{lipsum}
\usepackage{outlines}
\usepackage{float,tabularx}

\usepackage{algorithm,algcompatible}

\algnewcommand\algorithmicto{\textbf{to}}
\algnewcommand\RETURN{\STATE \textbf{return} }
\makeatletter
\newcommand{\multiline}[1]{%
    \begin{tabularx}{\dimexpr\linewidth-\ALG@thistlm}[t]{@{}X@{}}
        #1
    \end{tabularx}
}

\renewcommand*\E[1]{\mathbb{E}\left[#1\right]}

\usepackage{multirow}
\usepackage{subcaption}

\begin{document}

\maketitle
\renewcommand{\thefootnote}{\fnsymbol{footnote}}
\footnotetext[1]{\thanks{Authors contributed equally}}

\begin{abstract}
   Predicting how a dynamical unit evolves over time— how an individual ages, an epidemic spreads, or a physical system degrades— typically requires dense longitudinal tracking. When only extremely sparse or entirely cross-sectional data is available, inferring individualized, continuous-time trajectories is fundamentally ill-posed. Existing methods force a strict compromise: sequence models (e.g. latent ODEs) require dense {\it longitudinal} data, while cross-sectional methods (e.g. optimal transport, flow matching-based) map aggregate populations, losing {\it individual} dynamics. In this paper, we demonstrate that this dichotomy can be broken. We introduce CADENCE, a principled probabilistic framework that recovers continuous individual trajectories from isolated snapshots by anchoring latent dynamics to static, individual-level contexts. We provide novel identifiability guarantees for single-timepoint trajectory inference. By combining a score-based spatial encoder (bijective Probability Flow ODE) to eliminate diffeomorphic ambiguities with a Soft Mixture-of-Experts (SMoE) router, we show that individual dynamical parameters and routing function are jointly identifiable. Across a suite of benchmarks spanning physical systems to real-world biological data, CADENCE, trained strictly on extremely sparse snapshots with context structure, matches or exceeds the performance of state-of-the-art sequential models trained on dense, full-trajectory data.
\end{abstract}

\vspace{10pt}

\section{Introduction}
\label{sec:intro}

Understanding unit-level dynamics -- how individuals age, how patients progress through a disease, or how complex engineering systems degrade over time -- is a central goal across the quantitative sciences. The gold standard for capturing these continuous-time dynamics is dense longitudinal  data at the unit level, i.e. data in which the same unit (patient, mechanical system) is densely observed through time. However, in a wide range of real-world settings, 
such data may be difficult, expensive or even impossible to obtain, or  units may be prone to high attrition rates. As a result, much of the available large-scale data in a number of domains is \emph{very sparse} in terms of time-points per unit---including, among others,  longitudinal biomedical studies 
---where each unit is observed only a few times (possibly $\leq 3$ snapshots), often at irregular intervals. This extreme sparsity presents a fundamental dilemma: can we forecast individual, continuous-time trajectories without meaningful longitudinal measurements? 

\medskip
Mathematically, inferring an individual continuous vector field from unpaired cross-sectional data is a severely ill-posed inverse problem. 
Existing neural-dynamical methods face critical failure modes in this regime. Sequence models (e.g. Latent ODEs) need sufficiently long per-unit sequences to propagate gradients through time; with only 1--3 snapshots, they cannot learn meaningful temporal representations. Cross-sectional methods (optimal transport, flow matching) sidestep this by transporting aggregate distributions $P(t_1) \to P(t_2)$---but they ``track the crowd", losing {\it individual} information entirely. Contextual dynamics models (e.g. CoDA, LEADS) condition ODE parameters on individual context, but universally require dense sequential observations to fit the routing function: without per-unit trajectories, they have no mechanism to assign a new realization to a dynamical regime. 

\bigskip
\textbf{The core insight.} Temporal sparsity is  not necessarily a fundamental barrier if the ensemble has contextual structure. 
If such structure exists, a late-stage unit's trajectory can act as a temporal surrogate for a contextually similar early-stage unit. Then, pooling information across contextually matched individuals converts a sparse cross-section into an effective longitudinal signal. We translate this high-level intuition into CADENCE (Contextual Archetypes and Diffusion ENCodings for Dynamics Estimation), a unified methodological framework for learning latent individual trajectories from extremely sparse data. 
At one extreme, given a unit with a known contextual profile observed at only a {\it single} time point $t$, we ask whether it is possible to rigorously predict its future state at time $T$. An overview of the framework is shown in Figure~\ref{fig:overview}. Our main contributions are:
\begin{enumerate}[leftmargin=*]
  \item \textbf{Identifiability theory under explicit assumptions.} We provide identifiability guarantees for single-timepoint trajectory inference (Theorems~\ref{thm:foliation_id}--\ref{thm:composition}) under three structural assumptions grounded in
  foliation theory, nonlinear observability, and convex geometry. These constraints determine the CADENCE architecture: the score-based bijective Probability Flow ODE (PF-ODE) eliminates variational ambiguity while the SMoE router arises as a unique minimal implementation (rather than heuristic choice).
  \item \textbf{New learning paradigm.} The identifiability result establishes a new paradigm: individual continuous-time trajectory inference from $\leq 3$ observations per unit, formally bridging cross-sectional optimal transport and dense sequential modeling 
  (Table~\ref{tab:capability}).
  \item \textbf{Decoupled architecture \& strong empirical performance.} We introduce a decoupled two-stage training scheme that reduces computational cost from $O(\mathrm{Steps_{spatial}} \times \mathrm{Steps_{temporal}})$ to $O(\mathrm{Steps_{temporal}})$. Across seven benchmarks spanning epidemiology, ecology, and nonlinear physics, CADENCE, trained only on sparse snapshot data with context structure, matches or exceeds the performance of state-of-the-art (SOTA) sequential models trained on dense, full-trajectory datasets.
\end{enumerate}

\begin{table}[t]
\centering\small
\begin{tabular}{p{4cm}ccc}
\toprule
Method & Individual forecast & Context-dependent dynamics & Cross-sectional training \\
\midrule
OT / Flow matching$^{*}$        & \xmark & \xmark & \cmark \\
Latent ODE$^{\#}$                & \cmark & \xmark & \xmark \\
Context models$^{\dagger}$  & \cmark & \cmark & \xmark \\
\textbf{CADENCE (ours)}         & \cmark & \cmark & \cmark \\
\bottomrule
\end{tabular}
\caption{%
  Comparison of trajectory inference frameworks across three key desiderata.
  CADENCE is the only method satisfying all three simultaneously. Examples for comparison methods include
  $^{*}$\cite{tong2024otcfm,sha2024reconstructing,huguet2022manifold}, 
  $^{\#}$\cite{rubanova2019latent,lagemann2023invariance,iakovlev2023latent}, and
  $^{\dagger}$\cite{kirchmeyer2022generalizing,yin2021leads}.
}
\label{tab:capability}
\end{table}

%

\section{Problem Formulation}
\label{sec:problem}
We first formalize the generative model and discuss non-identifiability. This directly motivates the structural assumptions in Section~\ref{sec:model}. Full proofs are provided in Appendix~\ref{app:proofs}.

\bigskip 

\textbf{Latent dynamical model.}
We consider an ensemble of $N$ units/realizations, each observed once at an entry time $t_i \in \mathcal{T}$ yielding a measurement $\bx^i_{t_i} \in \R^p$ and a time-invariant context $\bc_i \in \R^r$. We assume the underlying dynamics unfold in a low-dimensional latent space $\R^q$ ($q \ll p$), governed by a neural ODE vector field $h_\theta$. Observations are then generated via a smooth spatial mapping $\mathcal{F}_\psi : \R^q \to \R^p$. Given a baseline time $t_0$ and an initial latent state $\bz^i_0$, individual trajectories satisfy
\begin{equation}
  \frac{d\bz^i_t}{dt} = h_\theta\!\left(\bz^i_t,\, t,\, \bc_i;\, \bw_i\right), \qquad \bz^i_{t_0} = \bz^i_0,
  \label{eq:ode}
\end{equation}
where $\bw_i \in \R^d$ is a unit-specific parameter vector (detailed in Section~\ref{sec:model}). Crucially, the cross-sectional dataset $\mathcal{D} = \{(\bx^i_{t_i}, t_i, \bc_i)\}_{i=1}^N$ only provides samples from the mixture measure $\mu = \int_\mathcal{T} \mu_t\, d\nu(t)$, a time-averaged blend of ensemble marginals rather than from any instantaneous marginal $\mu_t$ directly. Recovering the continuous-time vector field $h_\theta$ from this aggregate mixture $\mu$ constitutes a severely underconstrained inverse problem, reflecting well-known non-identifiability issues that we formalize for our setting in Proposition \ref{prop:nonid}:

\begin{proposition}[Non-identifiability without structure]
\label{prop:nonid}
  Given cross-sectional data $\mathcal{D}$, the set of model configurations consistent with perfect reconstruction is uncountably infinite. For any smooth diffeomorphism $\phi : \R^q \to \R^q$, an equivalent model defined by the tuple $(\tilde{h}, \mathcal{F}_\psi \circ \phi^{-1}, \{\phi(\bz^i_0)\})$ achieves identical reconstruction, where $\tilde{h}(\tilde{\bz}, t, \bc; \bw) = D\phi(\phi^{-1}(\tilde{\bz}))\, h_\theta(\phi^{-1}(\tilde{\bz}), t, \bc; \bw)$.
\end{proposition}

\noindent\emph{Proof.} Let $\tilde{\bz}^i_t = \phi(\bz^i_t)$. By the chain rule, $\tilde{\bz}^i_t$ satisfies the ODE with vector field $\tilde{h}$; the spatial mapping absorbs $\phi^{-1}$, leaving reconstruction unchanged. Since smooth diffeomorphisms on $\R^q$ form an infinite-dimensional group, valid configurations are uncountably infinite. (See Appendix~\ref{app:proofs}.)~\hfill$\square$\\

\noindent Rather than relying on informal ensemble averaging, we rigorously frame cross-sectional training as a measure-theoretic inverse problem: we seek a latent vector field $h_\theta$ whose flow transports the baseline measure $p(\bz_0)$ to match empirical marginals across heterogeneous enrollment times. However, because matching pushforward measures does not uniquely determine the flow, Proposition~\ref{prop:nonid} establishes that structural assumptions are \emph{necessary}, not merely convenient. We resolve this fundamental non-identifiability by choosing $\mathcal{F}_\psi$ as a score-based PF-ODE, introducing a disciplined structural prior that immediately constrains the symmetry group:

\begin{proposition}[Reduced Spatial Ambiguity]
\label{prop:reduced_ambiguity}
  If $\mathcal{F}_\psi$ is a score-based PF-ODE trained via denoising score matching, the symmetry group acting on the latent space reduces from the full diffeomorphism group $\mathrm{Diff}(\R^q)$ to the compact subgroup $\mathcal{S}_\mathrm{iso}$ of $\mathcal{N}(\mathbf{0},\mathbf{I})$-preserving isometries. (Proof: Appendix~\ref{app:proofs})
\end{proposition}

\noindent Training via denoising score matching pushes the data distribution bijectively to $\mathcal{N}(\mathbf{0},\mathbf{I})$, so latent dynamics are invariant only to isometric transformations within the Gaussian-pinned latent space which collapses spatial ambiguity from an infinite-dimensional group to a compact subgroup. However, while necessary this is {\it not} sufficient: the \emph{temporal} inverse problem remains open. Loosely speaking, recovering the ODE parameters $\bw_i$ from cross-sectional snapshots requires that the \emph{ensemble} carry enough temporal signal. Next, we introduce three structural assumptions that close this gap.

\definecolor{cadEncBlue}{HTML}{284E79}   
\definecolor{cadRoutGold}{HTML}{A8761C} 
\definecolor{cadFilmOrg}{HTML}{F59038}

\begin{figure}[t]
\centering
\begin{tikzpicture}
  \node[anchor=south west, inner sep=0] (image) at (0,0) {%
    \includegraphics[width=\linewidth, trim=0 2.6cm 0 2.1cm, clip]{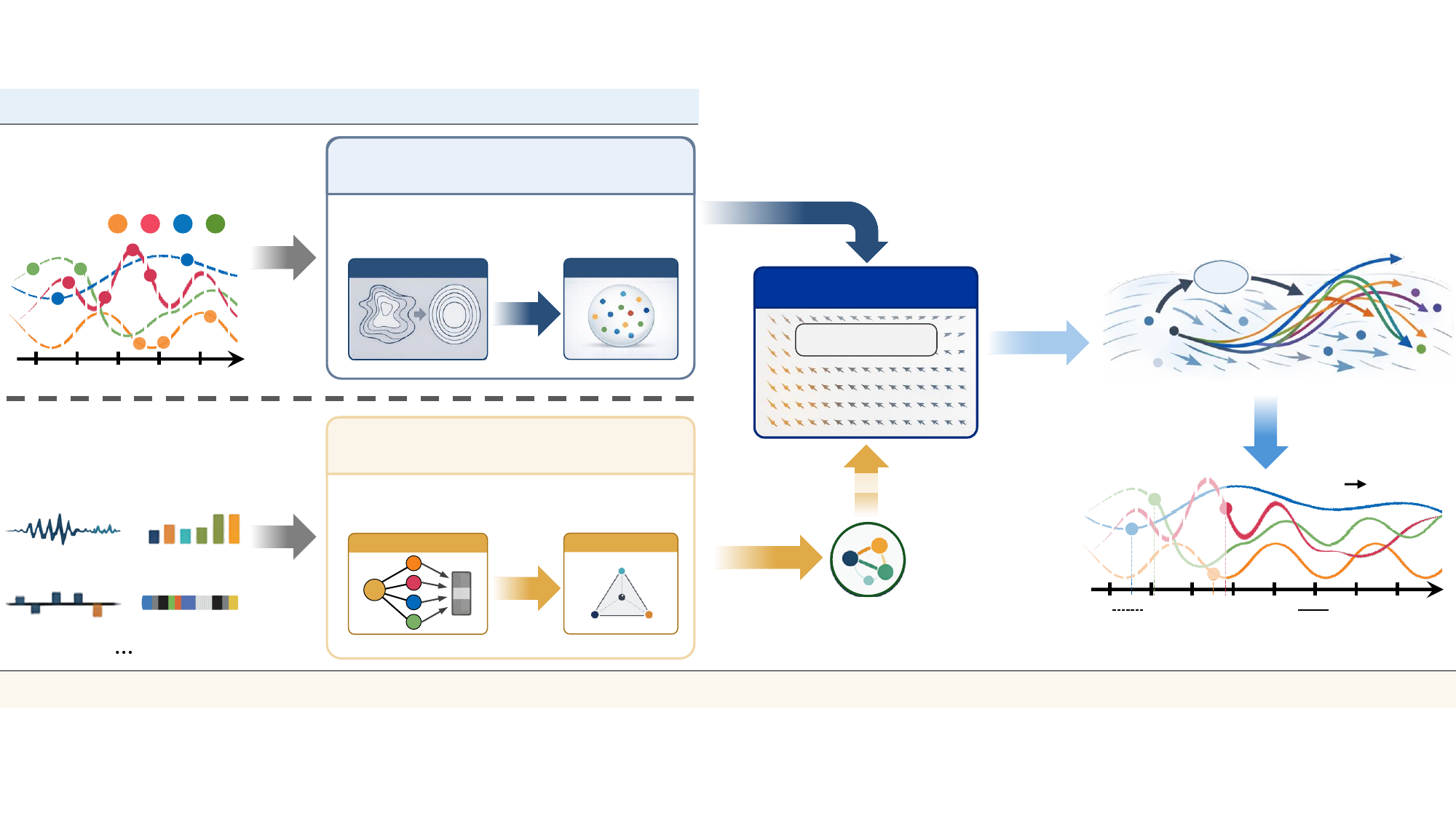}%
  };
  \begin{scope}[x={(image.south east)},y={(image.north west)}]

    \node[font={\fontsize{8}{9}\selectfont\bfseries}]
      at (0.240,0.97) {Stage 1: Spatial encoding};
    \node[font={\fontsize{8}{9}\selectfont\bfseries}]
      at (0.503,0.026) {Stage 2: Temporal Dynamics};

    \node[font={\fontsize{7}{8}\selectfont},align=center]
      at (0.086,0.871) {Sparse cross-sectional\\observation $X^i_{t_i}$};
    \node[font={\fontsize{5.7}{7}\selectfont},anchor=west]
      at (0.00,0.783) {Subgroups:};
    \node[font={\fontsize{5}{6}\selectfont},color=white] at (0.081,0.784) {0};
    \node[font={\fontsize{5}{6}\selectfont},color=white] at (0.103,0.784) {1};
    \node[font={\fontsize{5}{6}\selectfont},color=white] at (0.126,0.784) {2};
    \node[font={\fontsize{5}{6}\selectfont},color=white] at (0.148,0.784) {3};
    \node[font={\fontsize{6}{7}\selectfont}] at (0.091,0.533) {time};

    \node[font={\fontsize{7}{8}\selectfont}]
      at (0.09,0.428) {Static context $c_i$};
    \node[font={\fontsize{6}{6}\selectfont},align=center]
      at (0.045,0.356) {system\\signature};
    \node[font={\fontsize{6}{6}\selectfont},align=center]
      at (0.131,0.356) {individual\\statistics};
    \node[font={\fontsize{6}{6}\selectfont},align=center]
      at (0.042,0.214) {risk scores};
    \node[font={\fontsize{6}{6}\selectfont}]
      at (0.135,0.214) {baseline features};

    \node[font={\fontsize{7}{8}\selectfont\bfseries},color=cadEncBlue]
      at (0.350,0.876) {Score-based bijective encoder};
    \node[font={\fontsize{6.0}{6.5}\selectfont\itshape},anchor=west]
  at (0.231,0.801)
  {\textcolor{cadEncBlue}{$\bullet$}\enspace maps observations to latent space};
\node[font={\fontsize{6.0}{6.5}\selectfont\itshape},anchor=west]
  at (0.231,0.766)
  {\textcolor{cadEncBlue}{$\bullet$}\enspace removes spatial ambiguity};
    \node[font={\fontsize{4.0}{4.5}\selectfont},color=white]
      at (0.288,0.712) {Score-based PF-ODE};
    \node[font={\fontsize{4.0}{4.5}\selectfont},color=white]
      at (0.427,0.712) {Latent space};

    \node[font={\fontsize{7}{8}\selectfont\bfseries},color=cadRoutGold]
      at (0.352,0.423) {Context-dependent routing};
    \node[font={\fontsize{6.0}{6.5}\selectfont\itshape},anchor=west]
  at (0.231,0.328) {\textcolor{cadRoutGold}{$\bullet$}\enspace infers dynamical regime};
    \node[font={\fontsize{4.0}{4.5}\selectfont}]
      at (0.29,0.267) {SMoE router};
    \node[font={\fontsize{4.0}{4.5}\selectfont}]
      at (0.427,0.267) {Expert mix};
    \node[font={\fontsize{5.5}{6.5}\selectfont}]
      at (0.249,0.226) {$c_i$};
    \node[font={\fontsize{5.5}{6.5}\selectfont}]
      at (0.318,0.134) {$\alpha_i$};
    \node[font={\fontsize{4.0}{4.0}\selectfont}]
      at (0.43,0.133) {$w_i = W\alpha_i$};
    \node[font={\fontsize{3.3}{4.5}\selectfont}]
      at (0.427,0.24) {Convex combinations};

    \node[font={\fontsize{8}{9}\selectfont\bfseries},color=white]
      at (0.595,0.677) {Neural ODE};
    \node[font={\fontsize{4.8}{5.5}\selectfont}]
      at (0.595,0.595) {$h_\theta(z,t,c;w_i)$};

    \node[font={\fontsize{7}{8}\selectfont\bfseries},color=cadFilmOrg]
      at (0.589,0.363) {FiLM conditioning};

    \node[font={\fontsize{7.0}{7.5}\selectfont\itshape}]
      at (0.875,0.777) {Forward integration -- predict trajectory};
    \node[font={\fontsize{4.2}{5.5}\selectfont}]
      at (0.838,0.696) {$z_i(t_i)$};
    \node[font={\fontsize{4.2}{5.5}\selectfont}]
      at (0.793,0.603) {$z_{i0}$};

    \node[font={\fontsize{6.5}{6.5}\selectfont}]
      at (0.909,0.362) {$X^i_{t_i}$};
    \node[font={\fontsize{6.5}{6.5}\selectfont}]
      at (0.955,0.362) {$X^i_{T_i}$};
    \node[font={\fontsize{6.0}{6}\selectfont}]
      at (0.82,0.158) {$t_{enroll}$};
    \node[font={\fontsize{6.0}{6}\selectfont}]
      at (0.940,0.158) {$t_{pred}$};
    \node[font={\fontsize{7}{8}\selectfont\itshape},align=center]
      at (0.869,0.104) {Forecasted trajectories};

    \node[font={\fontsize{7}{8}\selectfont\itshape},align=center]
      at (0.615,0.115) {Pooling across similar contexts\\enables learning from sparse data};

  \end{scope}
\end{tikzpicture}
\caption{CADENCE overview. \textbf{Stage 1} maps each high-dimensional observation $\bx^i_{t_i}$ to a latent state $\bz^i_{t_i}$ via a score-based bijective PF-ODE, Gaussian-pinning the latent space. \textbf{Stage 2} routes the realization's static context $\bc_i$ through the SMoE gating network to produce a convex expert mixture $\bw_i$, which conditions the Neural ODE. Forward integration yields individual future trajectories.}
\label{fig:overview}
\end{figure}
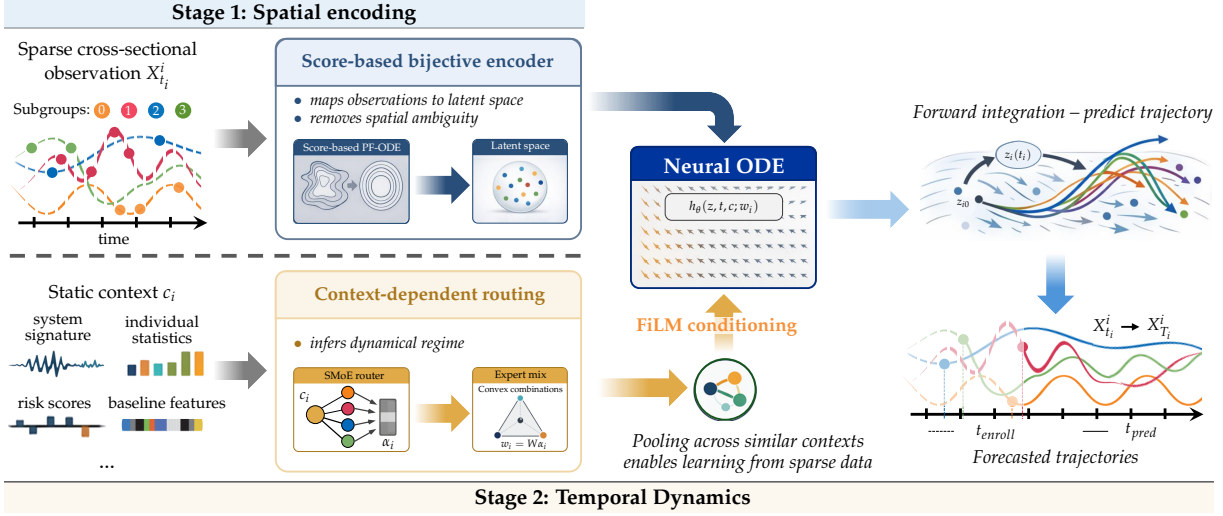
\bigskip
\section{Structural Assumptions and Heterogeneity Model}
\label{sec:model}
As outlined above, cross-sectional data is individually sparse but can be collectively rich. Exploiting this notion requires three things to hold simultaneously: (i) the ensemble exhibits shared dynamics (whether as a single group or partitioned into subgroups), (ii) any such group membership is partially encoded in static context, and (iii) the space of dynamical parameters is finite-dimensional and geometrically regular. We state each assumption, show why the previous one leaves a gap, and preview its architectural consequence. \textbf{Together they uniquely determine the CADENCE architecture --- not as design choices, but as the minimal implementation of these mathematical constraints.}

\medskip
\begin{assumption}[Dynamical Foliation Assumption, DFA]
\label{asm:dfa}
  There exists a measurable partition $\{L_\lambda\}_{\lambda\in\Lambda}$ of the context space $\mathcal{C}$, indexed by $\Lambda$, such that any two units $i,j$ with $c_i,c_j\in L_\lambda$ share the same dynamical parameter: $w_i = w_j = w(\lambda)$.
\end{assumption}

\noindent Intuitively, DFA partitions the ensemble into dynamical groups: all individuals in a leaf $L_\lambda$ follow the same ODE parameters $w(\lambda)$, so a late-stage realization's observation supervises the future of a contextually matched early-stage unit, converting sparse snapshots into an effective longitudinal signal. Formal grounding in the Rokhlin disintegration theorem and the Liouville equation $\rho^\lambda_t = (\Phi^{w(\lambda)}_{t,t_0})_\#\rho^\lambda_0$ (where $\rho^\lambda_t$ is the leaf-conditional density at time $t$ being pushed forward) is given in Appendix~\ref{app:proofs}.

\medskip
\begin{assumption}[Context Observability Assumption, COA]
\label{asm:coa}
  There exists a measurable function $g^*:\mathcal{C}\to\Lambda$ such that $\lambda(i) = g^*(c_i)$ for all $i \in \{1, \ldots,N\}$.
\end{assumption}

\noindent COA is the observability analogue from nonlinear control theory~\citep{herrmann1977controllability,isidori1985nonlinear}: static context must carry enough information to identify each realization's dynamical regime. Importantly, $g^*(c_i)$ provides only the \emph{leaf index} $\lambda(i)$, not the ODE parameters $w(\lambda)$ themselves. Recovering those parameters from cross-sectional data is the non-trivial content of Theorem~\ref{thm:foliation_id}, extending the identifiability of iVAE~\citep{khemakhem2020variational} to the dynamical regime.

\medskip
\begin{assumption}[Manifold Regularity Assumption, MRA]
\label{asm:mra}
  The parameter image $\mathcal{W}\!=\!\{w(g^*(c)):c\in\mathcal{C}\}\subset\R^d$ lies on a compact $d'$-dimensional submanifold ($d'\!\leq\!2r$, Whitney bound~\citep{whitney1944self}). Its convex hull is $\epsilon$-approximated by $K$ extreme points $\{w_1,\ldots,w_K\}\subset\mathcal{W}$:
  \begin{equation}
    w_i \approx \textstyle\sum\nolimits_{k=1}^K \alpha_{i,k}\,w_k, \quad \alpha_i\in\Delta^{K-1}, \quad \text{error}\leq\epsilon.
    \label{eq:cara}
  \end{equation}
\end{assumption}

\noindent MRA implies the continuous spectrum of individual dynamics is not arbitrarily complex, but rather a blend of fundamental ``archetypal'' behaviors. Because the neural mapping from an $r$-dimensional context is highly nonlinear, the resulting parameter manifold can twist and fold in $\R^d$. The Whitney embedding theorem safely bounds this topological complexity to an effective dimension of $d' \leq 2 r$. Since geometrically spanning a $d'$-dimensional space requires $d'+1$ extreme points, choosing $K\geq2r+1$ experts mathematically guarantees sufficient capacity to span the parameter manifold's convex hull (hence the factorisation is not simply a heuristic architectural choice). For continuous parameter manifolds, Choquet's theorem guarantees that any point in a compact convex set is the barycenter of a measure on its extreme points, while Carathéodory’s theorem ensures this mixture is finite. Our SMoE routing is therefore formally justified as a finite-dimensional discretization of this integral, providing a theoretically grounded representation of the dynamics (see Appendix~\ref{app:proofs} for full derivation and further discussions).

\newpage
\section{Architectural Design and Training Objective}
\label{sec:training}
Implementing the theoretical framework from Section~\ref{sec:model} directly parameterizes CADENCE's architecture. DFA and COA mandate that dynamic routing depends \emph{only} on the static context, while MRA bounds this routing to a convex combination over a shared expert basis. Finally, Proposition~\ref{prop:reduced_ambiguity} resolves spatial ambiguity via a score-based PF-ODE. This yields three core components:

\paragraph{Contextual SMoE.} The context router is implemented as an MLP $g_\omega : \R^r \to \R^K$. To satisfy MRA simplex constraints, individual routing weights are computed as $\balpha_i = \Softmax(g_\omega(\bc_i)) \in \Delta^{K-1}$. This smoothly blends a shared expert basis $\bW =[\bw_1, \dots, \bw_K] \in \R^{d \times K}$ to produce individual trajectory parameters $\bw_i = \bW\balpha_i$. Setting $\balpha_i$ to a one-hot vector recovers hard DFA leaf assignments as a discrete special case.

\paragraph{FiLM Conditioning.} The expert mixture $\bw_i$ conditions the ODE via Feature-wise Linear Modulation~\citep{perez2018film}: $h_\theta(\bz, t, \bc; \bw_i) = \gamma(\bw_i) \odot \tilde{h}(\bz,t,\bc) + \beta(\bw_i)$. Because the Whitney bound guarantees a low-dimensional parameter manifold, linear modulation of a shared backbone provides sufficient capacity to span all subgroups without requiring heavier mechanisms like hypernetworks.

\paragraph{Score-Based PF-ODE.} Instantiating the spatial encoder $\mathcal{F}_\psi$ exactly as a score-based PF-ODE \citep{song2021score} collapses residual spatial ambiguity to $\mathcal{S}_\mathrm{iso}$, which eliminates the need for variational approximations (or an ELBO) and avoids the characteristic blurring of VAE-based approaches. 

\paragraph{Decoupled Training Strategy.}
End-to-end backpropagation through cascaded spatial and temporal ODEs incurs a large $O(\mathrm{Steps_{spatial}} \times \mathrm{Steps_{temporal}})$ computational cost. We eliminate this bottleneck by splitting  training into two distinct stages.\\

\noindent \textit{Stage 1: Spatial Encoding.} The spatial network $\mathcal{F}_\psi$ is trained via denoising score matching on $(\bx, \bc_{\mathrm{static}})$ pairs. Once trained, $\mathcal{F}_\psi$ is frozen. All observations are then encoded offline using the deterministic PF-ODE to yield latents $\bz^i_{t_i} = \mathcal{F}_\psi(\bx^i_{t_i}; \bc_i)$, confining all spatial gradients to this stage.\\

\noindent \textit{Stage 2: Temporal Dynamics.} The temporal modules $(h_\theta, \bW)$ are trained entirely in the low-dimensional $\mathcal{Z}$-space at a vastly reduced $O(\mathrm{Steps_{temporal}})$ cost via forward integration:
\begin{equation}
  \hat{\bz}^i_t = \bz^i_{t_i} + \int_{t_i}^{t} h_\theta(\bz^i_\tau, \tau, \bc_i; \bw_i)\,d\tau, \quad t \geq t_i.
  \label{eq:forward_ode}
\end{equation}
During Stage 2, the vector field is supervised by one of two objectives depending on data availability:\\

\noindent \textit{Semi-longitudinal objective.} When sparse follow-up observations $\bx^i_{t_m}$ are available at times $t_m > t_i$, the maximum likelihood objective simplifies to exact latent MSE. Since $\mathcal{F}_\psi$ is a measure-preserving bijection, maximizing data likelihood $p(\bx)$ equals maximizing encoded-state likelihood $p(\bz)$:
\begin{equation}
  \mathcal{L}_{\mathrm{semi\text{-}long}} = \frac{1}{N} \sum_{i=1}^N \sum_{m} \left\| \mathcal{F}_\psi(\bx^i_{t_m}; \bc_i) - \hat{\bz}^i_{t_m} \right\|^2.
  \label{eq:latent_mse}
\end{equation}

\noindent \textit{Cross-sectional distributional objective.} For purely cross-sectional ensembles, consistency must be enforced \emph{per subgroup} to preserve the identifiability of context-driven dynamics. Using the SMoE routing weights $\alpha_{i,k}$, we define the model's predicted conditional marginal for the $k$-th archetype at query time $t$, alongside a Gaussian-kernel-smoothed ($\kappa_h$) reference marginal:
\begin{equation}
\hat{\mu}^{(k)}_t \propto \sum_{i: t_i \leq t} \alpha_{i,k} \, \delta_{\hat{\bz}^i_t}, \quad \text{and} \quad \mu^{\mathrm{ref}, (k)}_t \propto \sum_{j=1}^N \alpha_{j,k} \, \kappa_h(t_j - t) \, \delta_{\bz^j_{t_j}}.
\end{equation}

\noindent The cross-sectional loss minimizes the continuous-time discrepancy jointly across all $K$ subgroups:
\begin{equation}
\mathcal{L}_{\mathrm{cross}} = \mathbb{E}_{t \sim \mathcal{U}(t_0, t_{\mathrm{end}})}\left[ \sum_{k=1}^K \mathrm{MMD}^2\big(\hat{\mu}^{(k)}_t, \mu^{\mathrm{ref}, (k)}_t\big) \right].
\label{eq:cross_loss}
\end{equation}
This objective may be tractably approximated via Monte Carlo sampling, and it provides a rigorous ensemble-level consistency constraint on the vector field without requiring matched temporal pairs.

\medskip 
\textbf{Full Composite Objective.} The final Stage 2 objective combines the primary dynamics loss with structural regularization: $$\mathcal{L} = \mathcal{L}_{\mathrm{dyn}} + \lambda_{\mathrm{router}}\mathcal{L}_{\mathrm{router}} + \lambda_{\mathrm{feat}}\mathcal{L}_{\mathrm{feat}},$$ where $\mathcal{L}_{\mathrm{dyn}} \in \{\mathcal{L}_{\mathrm{semi\text{-}long}}, \mathcal{L}_{\mathrm{cross}}\}$. Standard entropy regularization ($\mathcal{L}_{\mathrm{router}}$) encourages diverse expert utilization, while optional penalties ($\mathcal{L}_{\mathrm{feat}}$) can promote interpretable disentanglement. Temperature annealing ($\tau \to 0$) is used to gradually recover hard DFA leaf assignments during training.

\bigskip
\section{Theoretical Analysis}
\label{sec:theory}

We now prove that the three structural assumptions of Section~\ref{sec:model} --- together with the spatial prior of Section~\ref{sec:problem} --- are sufficient to recover the full generative system from cross-sectional data. Our argument proceeds in stages, with each step closing a specific identifiability gap:\\

\textbf{1. Resolving Spatial Ambiguity.} Proposition~\ref{prop:nonid} established that without structure, infinitely many models fit any dataset. Proposition~\ref{prop:reduced_ambiguity} resolved this by pinning the spatial marginal to $\mathcal{N}(\mathbf{0},\mathbf{I})$, thereby collapsing spatial ambiguity to the compact subgroup $\mathcal{S}_{\mathrm{iso}}$. \\

\textbf{2. Resolving Temporal Ambiguity.} What remains is identifying the ODE parameters $w(\lambda)$ for each leaf. Theorem~\ref{thm:foliation_id} resolves this using our core structural assumptions (DFA, COA, and MRA). Theorem~\ref{thm:composition} then establishes that the full composite system is identifiable in the quotient $\mathcal{M}/\mathcal{S}_{\mathrm{iso}}$.\\

\textbf{3. Regularity Conditions.} Two regularity conditions complete the proof chain: \emph{Parametric Restriction} (PR, Assumption~\ref{asm:pr}, Appendix \ref{app:foa_conditions}) restricts $h_\theta$ to a finite-dimensional family so the inverse transport problem is well-posed, and \emph{Flow Observability} (FOA, Assumption~\ref{asm:foa}, Appendix \ref{app:foa_conditions}) requires that distinct parameters produce observationally distinct pushforward families. \\

Both PR and FOA are naturally satisfied by neural ODEs with generic initialisation and sufficiently diverse enrollment times. Notably, our BM7 experiment (see below) stress-tests FOA: when parameters are near-identical in their observable effects, recovery degrades gracefully and our model's routing entropy correctly reflects this uncertainty. Formal statements, sufficient conditions, and connections to persistent excitation in classical system identification are deferred to Appendix~\ref{app:proofs}.

\medskip
\begin{theorem}[Parameter identifiability, informal]
\label{thm:foliation_id}
Under Assumptions~\ref{asm:dfa}--\ref{asm:foa}, the per-leaf parameter $\hat{w}(\lambda) = \bW\alpha(\lambda)$ is  identifiable from the ensemble distribution $\{(\bx^i_{t_i}, \bc_i, t_i)\}$ up to leaf-label permutation, with approximation error $\leq \epsilon$ from the finite-$K$ Carath\'{e}odory approximation (MRA).
\end{theorem}

\textit{Key mechanism.} Bijectivity of $\mathcal{F}_\psi$ uniquely yields latent states; COA assigns leaf membership from context; within each leaf, diverse enrollment times constrain the pushforward family $\{(\Phi^{\hat{w}(\lambda)}_{t,t_0})_\# \rho^{\lambda}_0 : t \in \mathcal{T}_\lambda\}$, and FOA injectivity uniquely determines $\hat{w}(\lambda)$.

\medskip
\begin{theorem}[Composition identifiability, informal]
\label{thm:composition}
Under the conditions of Theorem~\ref{thm:foliation_id}, the full composite system $(\mathcal{F}_\psi, h_\theta, g^*)$ is identifiable in the quotient $\mathcal{M} / \mathcal{S}_\mathrm{iso}$: any two models consistent with the data differ only by an $\mathcal{N}(\mathbf{0},\mathbf{I})$-preserving diffeomorphism.
\end{theorem}

\textit{Implication.} Training Stage~1 and Stage~2 independently does not compromise identifiability of the underlying trajectories --- the residual spatial symmetry $\mathcal{S}_\mathrm{iso}$ is observationally harmless. Full formal statements, proof sketches, and supporting results (Corollary~\ref{cor:dfa_pooled} on pooled identifiability via the Liouville equation; Proposition~\ref{prop:anneal} on temperature annealing convergence; Appendix \ref{app:mmd_foa_bridge} proving identifiability of FOA with MMD) are in Appendix~\ref{app:proofs}.

\bigskip
\section{Related Work}
\label{sec:related_work}
Learning continuous individual trajectories from cross-sectional data sits at the intersection of generative modeling, latent dynamics, and identifiability theory. 
This goal faces a fundamental dilemma: existing methods either require dense individual sequences or sacrifice individual identity for population-level transport.\\

\textbf{Population-Level Transport.} Reconstructing dynamics from unpaired cross-sectional snapshots is traditionally framed as an optimal transport (OT) problem \citep{schiebinger2019optimal}. Recent advances in continuous normalizing flows (OT-CFM \citep{tong2024otcfm}) and unbalanced OT (TIGON \citep{sha2024reconstructing}) provide scalable, simulation-free objectives for cross-sectional data. However, these methods fundamentally compute aggregate mass transport; without an informational oracle (paired longitudinal endpoints), they scramble individual identities and cannot produce context-driven, individualized forecasts.\\

\textbf{Individualized Latent Dynamics.} Conversely, continuous-time latent state-space models—such as Latent ODEs \citep{chen2018neural} and context-adaptive extensions like CoDA \citep{kirchmeyer2022generalizing}, LaDID \citep{lagemann2023invariance}, and Neural ODE Processes \citep{norcliffe2021neural}—successfully model highly heterogeneous, individualized trajectories. Yet, they universally require \emph{dense, repeated} longitudinal measurements per unit to constrain latent initial conditions or infer context representations. While Soft Mixture-of-Experts \citep{puigcerver2024soft} provides flexible parameter generation, adapting it to deep latent ODEs typically demands sequential training data. CADENCE breaks this dichotomy, extending these routing paradigms to assemble individualized ODEs strictly from static covariates and sparse timepoint snapshots.\\

\textbf{Identifiability and Control Theory.} The central theoretical barrier to bridging these paradigms is latent identifiability \citep{locatello2019challenging}. While existing literature resolves this using auxiliary variables for static distributions \citep{khemakhem2020variational}, temporal constraints for fully-observed sequences \citep{yao2021learning}, causal structure in dynamical latent systems \citep{yao2024marrying}, or structured causal learning under high-dimensional data limitations \citep{lagemann2023deep}, formal guarantees for \emph{cross-sectional} dynamics remain absent. CADENCE addresses this gap by drawing an equivalence to classical nonlinear control theory \citep{herrmann1977controllability}. We prove that combining static context routing (satisfying observability rank conditions) with score-based encoders \citep{song2021score} strictly bounds latent ambiguity, providing the first exact identifiability guarantees for single-timepoint trajectory inference (Theorems \ref{thm:foliation_id} \& \ref{thm:composition}).

\bigskip
\section{Experiments and Baselines}
\label{sec:experiments}

\textbf{Benchmark suite.}
Evaluating individual-level trajectory inference from sparse cross-sectional snapshots presents a fundamental methodological challenge: real-world ensembles often lack the continuous, individualized longitudinal ground truth required for validation. To rigorously assess CADENCE, we use a comprehensive suite of seven benchmarks that progressively scale in observation dimensionality ($p \in [2, 500]$) and complexity up to challenging real-world applications: 
\begin{enumerate}[leftmargin=*, topsep=2pt, itemsep=2pt]
    \item \textbf{Controlled Dynamical Systems (BM1, BM3, BM5):} Ranging from ecology (Lotka--Volterra) to nonlinear physics (Van der Pol, Duffing), these low-dimensional environments provide exact ground-truth vector fields. They isolate specific theoretical failure modes, testing exact SMoE routing, continuous parameter manifold approximations (MRA), and identifiability limits near chaotic bifurcations.
    \item \textbf{High-Dimensional \& Constrained Biomedicine (BM2, BM4):} Simulating epidemiological ensembles (SIR models) and high-dimensional gene expression, these tasks stress-test domain-specific inductive biases (e.g., partial observability, monotonicity) and the stability of the Stage~1 PF-ODE encoding under severe spatial compression ($p \gg q$).
    \item \textbf{Primary Single-Cell Genomics (BM6, BM7):} Bridging deterministic theory and stochastic reality, BM6 utilizes the SERGIO simulator \citep{dibaeinia2020sergio} to generate trajectories governed by  gene regulatory networks (GRNs) and transcriptional bursting. BM7 serves as the primary real-world evaluation using LARRY \citep{weinreb2020lineage}—a landmark single-cell dataset of  haematopoiesis. By leveraging experimental DNA lineage tracing, BM7 provides a valuable, experimental gold-standard temporal evaluation target, allowing us to assess CADENCE's cross-sectionally trained predictions against true biological observations.
\end{enumerate}

\medskip

\textbf{Baselines.}
We compare against three groups: 
\begin{itemize}
    \item[] \textbf{(1)~Cross-sectional ensemble estimators} (scNODE, TIGON, OT-CFM), which learn a shared population-level flow from unpaired snapshots and cannot produce individualised forecasts. 
    \item[] \textbf{(2)~Contextual dynamics models} (NDP, CoDa), which condition on individual context but require longitudinal observation pairs for training.
    \item[] \textbf{(3)~Latent continuous-time models} (ODE-RNN, ODE$^2$VAE, LaDID), which encode sequential observations into a latent state and integrate forward. 
\end{itemize}
All baselines are trained on full longitudinal sequences; OT-CFM, TIGON and scNODE are additionally trained in pure cross-sectional settings. 
Additional baseline values (scDiffEq, MIOFlow, TrajectoryNet, PRESCIENT) are provided where available from the literature. 

\bigskip
\section{Results}
\label{sec:results}

All benchmarks represent continuous-time evolution from an origin $t_0$. At an arbitrary observation time $t_{obs}$, a single state measurement is recorded alongside static context $c_i$. Crucially, $t_{obs}$ is drawn randomly per realization and carries no dynamical signal. Because standard sequential models fail catastrophically when restricted to extremely sparse or cross-sectional data, we evaluate them using a deliberately \textbf{asymmetric comparison: All competitor models are trained on full, dense longitudinal trajectories. CADENCE, by contrast, is trained strictly on sparse data: it sees only sparse cross-sectional snapshots plus static context per unit.} Therefore, baseline performance represents an empirical upper bound of what standard models achieve with dense temporal data.

\paragraph{Metrics.}
We evaluate performance using two complementary metrics: \textbf{Mean Absolute Error (MAE)} and \textbf{Sliced Wasserstein Distance (SW$_2$)}. MAE is a strict individual-level metric measuring how closely each unit's predicted trajectory matches its own ground truth. Unlike transport-based baselines that match aggregate flows, MAE strictly penalizes instance-level prediction errors. SW$_2$ is a population-level metric capturing whether the overall distribution of predicted states matches the true system manifold at future time steps, detecting mode collapse or distributional drift.

\begin{table}[t]
\centering
\caption{MAE and SW$_2$ at prediction horizon $H{=}20$. 
\textbf{Note asymmetric data access:} to provide strong empirical upper bounds, all  baselines are given an informational advantage and trained on full longitudinal trajectories (\textit{italicized}), with the strongest oracle baseline (OT-CFM) \underline{\textit{underlined}}. CADENCE is trained on sparse cross-sectional data, yet achieves competitive performance.}
\label{tab:paper_mae_sw2_h20}
\resizebox{\linewidth}{!}{%
{\small
\begin{tabular}{ll cccc cccccccc}
\toprule
& & \multicolumn{4}{c}{\textbf{Cross-Sectional Data}} & \multicolumn{8}{c}{\textbf{Full Trajectory Data (Oracle)}} \\
\cmidrule(lr){3-6} \cmidrule(l){7-14}
\textbf{BM} & \textbf{Metric} & \textbf{CADENCE} & OT-CFM & TIGON & scNODE & OT-CFM & TIGON & scNODE & NDP & CoDA & ODE{-}RNN & ODE$^2$VAE & LaDID \\
\midrule
1 & MAE & \textbf{2.05} & 27.47 & 34.85 & 33.05 & \underline{\textit{3.29}} & \textit{11.40} & \textit{28.83} & \textit{7.15} & 57.95 & \textit{45.93} & \textit{18.67} & \textit{25.55} \\
\rowcolor{blue!6}  & SW$_2$ & \textbf{0.070} & 0.493 & 0.815 & 0.873 & \underline{\textit{0.044}} & \textit{0.108} & \textit{1.000} & \textit{0.109} & 2.000 & \textit{1.446} & \textit{0.867} & \textit{1.158} \\
\midrule
2 & MAE & 0.025 & \textbf{0.023} & 0.120 & 0.154 & \underline{\textit{0.000}} & \textit{0.008} & \textit{0.169} & \textit{0.155} & 0.109 & \textit{0.169} & \textit{0.600} & \textit{0.271} \\
\rowcolor{blue!6}  & SW$_2$ & \textbf{0.050} & 0.058 & 0.127 & 1.306 & \underline{\textit{0.003}} & \textit{0.043} & \textit{0.725} & \textit{0.462} & 0.508 & \textit{0.512} & \textit{2.097} & \textit{1.073} \\
\midrule
3 & MAE & \textbf{0.097} & 1.59 & 1.59 & 1.32 & \underline{\textit{0.293}} & \textit{0.518} & \textit{1.349} & \textit{1.112} & 1.676 & \textit{1.924} & \textit{0.142} & \textit{1.171} \\
\rowcolor{blue!6}  & SW$_2$ & \textbf{0.050} & 0.103 & 0.104 & 0.735 & \underline{\textit{0.036}} & \textit{0.044} & \textit{0.762} & \textit{0.218} & \textit{0.493} & \textit{0.462} & \textit{0.078} & \textit{0.824} \\
\midrule
4 & MAE & \textbf{0.061} & 0.072 & 0.099 & 0.200 & \underline{\textit{0.002}} & \textit{0.030} & \textit{0.173} & \textit{0.102} & \textit{0.152} & \textit{0.130} & \textit{0.214} & \textit{0.130} \\
\rowcolor{blue!6}  & SW$_2$ & 0.241 & \textbf{0.216} & 0.338 & 0.690 & \underline{\textit{0.007}} & \textit{0.065} & \textit{0.511} & \textit{0.376} & \textit{2.973} & \textit{0.471} & \textit{1.054} & \textit{0.616} \\
\midrule
5 & MAE & \textbf{0.103} & 0.236 & 0.622 & 0.303 & \underline{\textit{0.074}} & \textit{0.261} & \textit{0.433} & \textit{0.358} & \textit{0.289} & \textit{1.397} & \textit{0.234} & \textit{0.706} \\
\rowcolor{blue!6}  & SW$_2$ & \textbf{0.137} & 0.510 & 1.859 & 1.030 & \underline{\textit{0.035}} & \textit{0.088} & \textit{0.837} & \textit{0.350} & \textit{0.535} & \textit{2.756} & \textit{0.481} & \textit{1.173} \\
\bottomrule
\end{tabular}%
}
}
\vspace{0.8cm}
\end{table}

\paragraph{Quantitative comparison across benchmarks.}
As shown in Table~\ref{tab:paper_mae_sw2_h20}, CADENCE outperforms all cross-sectional baselines and matches or surpasses sequence-trained baselines using only isolated snapshots. On the core low-dimensional environments (BM1, BM3), CADENCE consistently achieves the lowest per-instance MAE. On systems with monotonic constraints (BM2) or high-dimensional observations (BM4), OT-CFM trained on full trajectory data achieves the lowest MAE by directly exploiting its access to paired sequence endpoints; however, CADENCE remains highly competitive and leads all other baselines. Crucially, CADENCE maintains excellent SW$_2$ scores across all settings, whereas sequential models (ODE-RNN, LaDID) frequently exhibit order-of-magnitude worse SW$_2$, indicating severe distributional drift despite their dense training data (these trends remain stable at longer prediction horizons). 

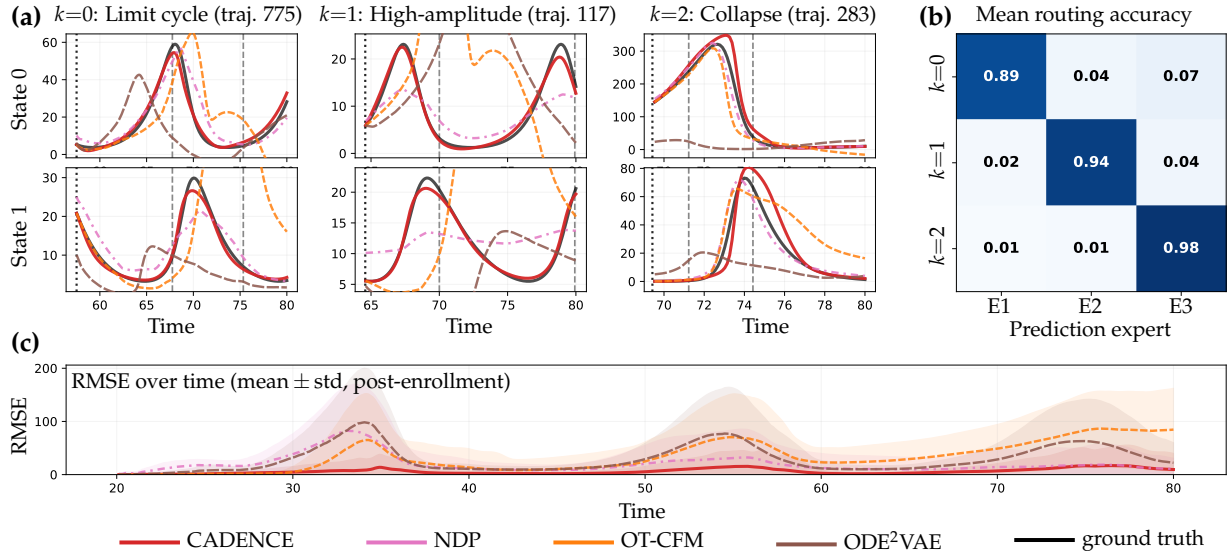
\begin{figure}[b]
    \centering
    \begin{tikzpicture}
        \node[anchor=south west, inner sep=0] (img) at (0,0) {%
            \includegraphics[width=\linewidth]{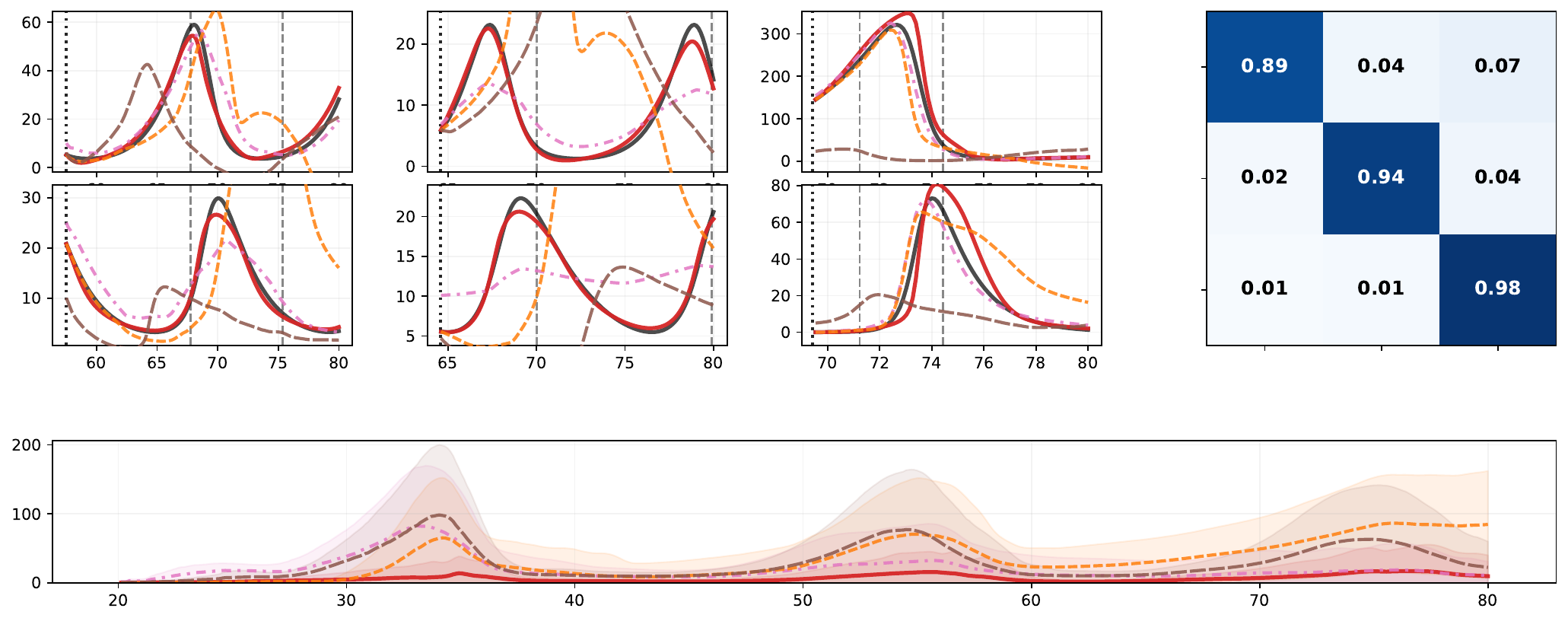}%
        };
        \begin{scope}[
            x={($(img.south east)-(img.south west)$)},
            y={($(img.north west)-(img.south west)$)}
        ]
            \node[anchor=north west, font=\small\bfseries]  at (-0.02,1.06) {(a)};
            \node[anchor=north west, font=\small\bfseries]  at (0.73,1.06) {(b)};
            \node[anchor=north, font=\scriptsize]       at (0.875,1.06) {Mean routing accuracy};
            \node[anchor=north, font=\scriptsize]  at (0.127,1.06)
                {$k{=}0$: Limit cycle (traj. 775)};
            \node[anchor=north, font=\scriptsize]  at (0.366,1.06)
                {$k{=}1$: High-amplitude\ (traj. 117)};
            \node[anchor=north, font=\scriptsize]  at (0.610,1.06)
                {$k{=}2$: Collapse (traj. 283)};

            \node[anchor=north west, font=\small\bfseries] at (-0.02,0.38) {(c)};

\node[anchor=east, font=\scriptsize, rotate=90, transform shape] at (0.753,0.96) {$k{=}0$};
\node[anchor=east, font=\scriptsize, rotate=90, transform shape] at (0.753,0.78) {$k{=}1$};
\node[anchor=east, font=\scriptsize, rotate=90, transform shape] at (0.753,0.6) {$k{=}2$};

            \node[anchor=north, font=\scriptsize] at (0.805,0.45) {E1};
            \node[anchor=north, font=\scriptsize] at (0.88,0.45) {E2};
            \node[anchor=north, font=\scriptsize] at (0.955,0.45) {E3};

            \node[anchor=north, font=\scriptsize] at (0.88,0.405) {Prediction expert};

            \node[rotate=90, anchor=south, font=\scriptsize] at (0.01,0.845) {State 0};
            \node[rotate=90, anchor=south, font=\scriptsize] at (0.01,0.56) {State 1};

            \node[anchor=north, font=\scriptsize] at (0.121,0.41) {Time};
            \node[anchor=north, font=\scriptsize] at (0.367,0.41) {Time};
            \node[anchor=north, font=\scriptsize] at (0.612,0.41) {Time};

            \node[anchor=north, font=\scriptsize] at (0.22,0.295)
                {RMSE over time (mean\,$\pm$\,std, post-enrollment)};

            \node[rotate=90, anchor=south, font=\scriptsize] at (0.01,0.163) {RMSE};

            \node[anchor=south, font=\scriptsize] at (0.510,-0.05) {Time};

        \definecolor{cadencecolor}{HTML}{d62728}
        \definecolor{ndpcolor}{HTML}{e377c2}
        \definecolor{cfmcolor}{HTML}{ff7f0e}
        \definecolor{ode2vaecolor}{HTML}{8c564b}
        \definecolor{gtcolor}{HTML}{000000}
        
        \matrix[
            anchor=west,
            row sep=0pt,
            column sep=1.75em,
        ] at (0.06,-0.075)
        {
            \node[font=\scriptsize]{\tikz[baseline=-0.5ex]{\draw[cadencecolor, line width=1.5pt] (0,0)--(0.05,0);} CADENCE}; &
            \node[font=\scriptsize]{\tikz[baseline=-0.5ex]{\draw[ndpcolor, line width=1.5pt] (0,0)--(0.05,0);} NDP}; &
            \node[font=\scriptsize]{\tikz[baseline=-0.5ex]{\draw[cfmcolor, line width=1.5pt] (0,0)--(0.05,0);} OT-CFM}; &
            \node[font=\scriptsize]{\tikz[baseline=-0.5ex]{\draw[ode2vaecolor, line width=1.5pt] (0,0)--(0.05,0);} ODE$^2$VAE}; &
            \node[font=\scriptsize]{\tikz[baseline=-0.5ex]{\draw[gtcolor, line width=1.5pt] (0,0)--(0.05,0);} ground truth}; \\
        };
        
        \end{scope}
    \end{tikzpicture}

    \caption{BM1  results comparing CADENCE against published baselines.
    \textbf{(a)}~Per-expert predicted trajectories  vs.\ ground truth.
    \textbf{(b)}~Mean routing accuracy matrix: diagonal confirms effective subgroup prediction.
    \textbf{(c)}~ RMSE across the prediction horizon.
    }
    \label{fig:bm1_combined}
\end{figure}

\paragraph{BM1: Trajectory recovery and subgroup routing.}
BM1 isolates the core challenge of individual trajectory inference. Comparing against the best-performing baselines (OT-CFM, NDP, ODE$^2$VAE), Figure~\ref{fig:bm1_combined}a shows that CADENCE closely tracks ground-truth trajectories across all subgroup dynamics regimes. In contrast, standard sequence models either collapse to a single sub-population (e.g., k=2) or diverge entirely when faced with heterogeneous dynamics.
Panel~(c) confirms that CADENCE maintains the tightest error band across the full prediction horizon. Finally, panel~(b) demonstrates that CADENCE's SMoE router recovers the true subgroup structure in a fully unsupervised manner, achieving 94\% routing accuracy.

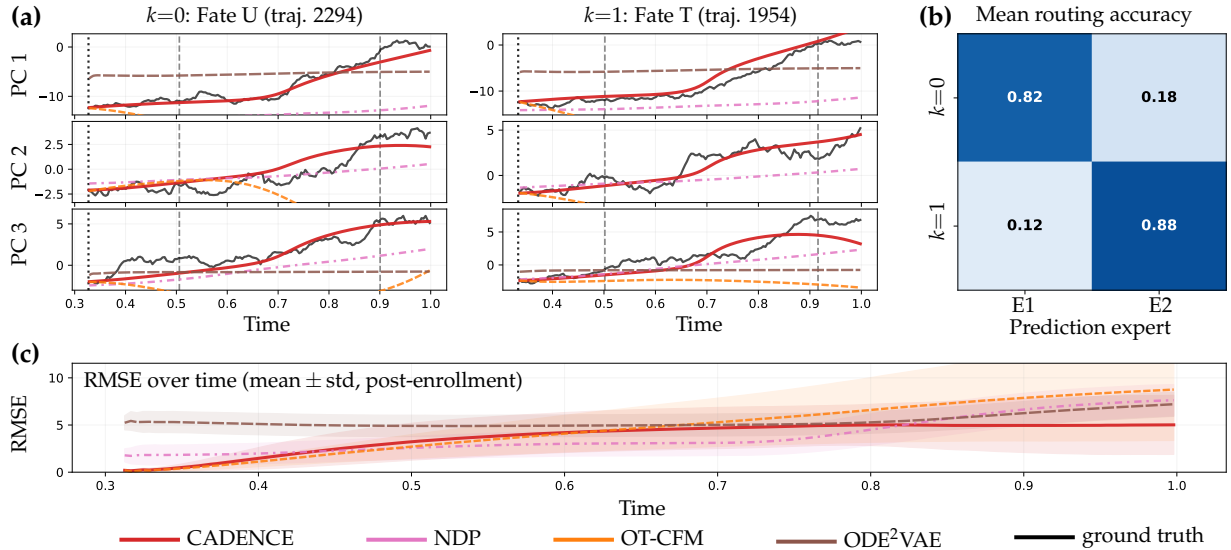
\begin{figure}[b]
    \centering
    \begin{tikzpicture}
        \node[anchor=south west, inner sep=0] (img) at (0,0) {%
            \includegraphics[width=\linewidth]{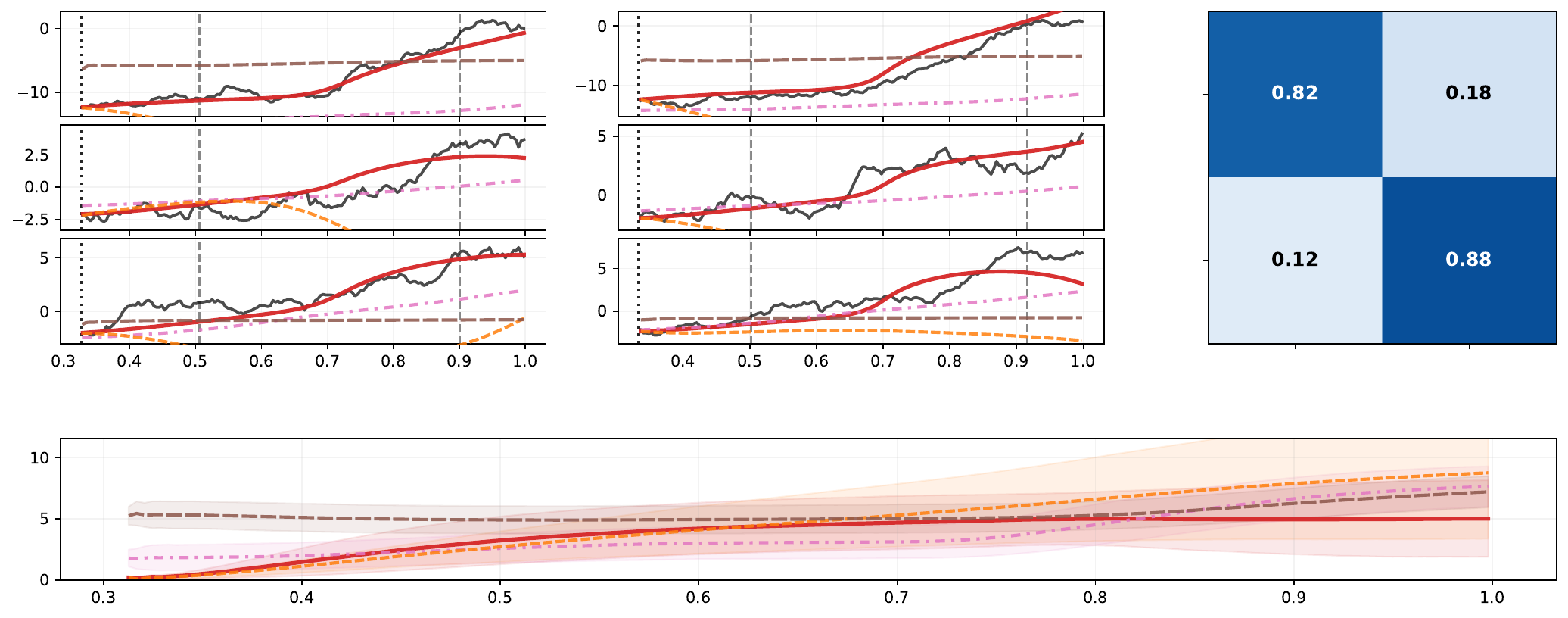}%
        };
        \begin{scope}[
            x={($(img.south east)-(img.south west)$)},
            y={($(img.north west)-(img.south west)$)}
        ]
            \node[anchor=north west, font=\small\bfseries]  at (-0.02,1.06) {(a)};
            \node[anchor=north west, font=\small\bfseries]  at (0.73,1.06) {(b)};
            \node[anchor=north, font=\scriptsize]       at (0.875,1.06) {Mean routing accuracy};
            \node[anchor=north, font=\scriptsize]  at (0.19,1.06)
                {$k{=}0$: Fate U (traj. 2294)};
            \node[anchor=north, font=\scriptsize]  at (0.55,1.06)
                {$k{=}1$: Fate T\ (traj. 1954)};

            \node[anchor=north west, font=\small\bfseries] at (-0.02,0.355) {(c)};

\node[anchor=east, font=\scriptsize, rotate=90, transform shape] at (0.753,0.91) {$k{=}0$};
\node[anchor=east, font=\scriptsize, rotate=90, transform shape] at (0.753,0.65) {$k{=}1$};
            \node[anchor=north, font=\scriptsize] at (0.825,0.45) {E1};
            \node[anchor=north, font=\scriptsize] at (0.938,0.45) {E2};

            \node[anchor=north, font=\scriptsize] at (0.88,0.405) {Prediction expert};

            \node[rotate=90, anchor=south, font=\scriptsize] at (0.01,0.895) {PC 1};
            \node[rotate=90, anchor=south, font=\scriptsize] at (0.01,0.70) {PC 2};
            \node[rotate=90, anchor=south, font=\scriptsize] at (0.01,0.54) {PC 3};

            \node[anchor=north, font=\scriptsize] at (0.2,0.41) {Time};
            \node[anchor=north, font=\scriptsize] at (0.557,0.41) {Time};

            \node[anchor=north, font=\scriptsize] at (0.23,0.295)
                {RMSE over time (mean\,$\pm$\,std, post-enrollment)};

            \node[rotate=90, anchor=south, font=\scriptsize] at (0.01,0.163) {RMSE};

            \node[anchor=south, font=\scriptsize] at (0.510,-0.05) {Time};

        \definecolor{cadencecolor}{HTML}{d62728}
        \definecolor{ndpcolor}{HTML}{e377c2}
        \definecolor{cfmcolor}{HTML}{ff7f0e}
        \definecolor{ode2vaecolor}{HTML}{8c564b}
        \definecolor{gtcolor}{HTML}{000000}
        
        \matrix[
            anchor=west,
            row sep=0pt,
            column sep=1.75em,
        ] at (0.06,-0.075)
        {
            \node[font=\scriptsize]{\tikz[baseline=-0.5ex]{\draw[cadencecolor, line width=1.5pt] (0,0)--(0.05,0);} CADENCE}; &
            \node[font=\scriptsize]{\tikz[baseline=-0.5ex]{\draw[ndpcolor, line width=1.5pt] (0,0)--(0.05,0);} NDP}; &
            \node[font=\scriptsize]{\tikz[baseline=-0.5ex]{\draw[cfmcolor, line width=1.5pt] (0,0)--(0.05,0);} OT-CFM}; &
            \node[font=\scriptsize]{\tikz[baseline=-0.5ex]{\draw[ode2vaecolor, line width=1.5pt] (0,0)--(0.05,0);} ODE$^2$VAE}; &
            \node[font=\scriptsize]{\tikz[baseline=-0.5ex]{\draw[gtcolor, line width=1.5pt] (0,0)--(0.05,0);} ground truth}; \\
        };
        
        \end{scope}
    \end{tikzpicture}

    \caption{BM6  results comparing CADENCE against published baselines.
    \textbf{(a)}~Per-expert predicted trajectories  vs.\ ground truth.
    \textbf{(b)}~Mean routing accuracy matrix: diagonal confirms effective  subgroup prediction.
    \textbf{(c)}~ RMSE across the prediction horizon.
    }
    \label{fig:sergio_combined}
\end{figure}


\paragraph{BM6--BM7: High-dimensional and real-world stochastic dynamics.}
The final benchmarks extend to high-dimensional stochastic dynamics. BM6 models the nonlinear evolution of a gene regulatory network ($p=100$, $K=2$). Despite intrinsic stochasticity, CADENCE closely matches OT-CFM in both MAE ($0.434$ vs. $0.417$) and RMSE over time (Fig.~\ref{fig:sergio_combined}c). Crucially, only CADENCE faithfully recovers both subgroup trajectories: all other methods either collapse to the ensemble mean or approximate a single subgroup, as shown in Fig.~\ref{fig:sergio_combined}a. Crucially, CADENCE further determines cell subgroups precisely achieving a 85\% routing accuracy (Fig.~\ref{fig:sergio_combined}b). The distributional comparison reinforces this picture: CADENCE and OT-CFM achieve SW$_2$ scores of $0.20$ and $0.12$, respectively, whereas \emph{all other methods exceed $1.0$}, indicating severe distributional divergence from the true cell population. Together, these results demonstrate that CADENCE's cross-sectional training framework remains effective under the high-dimensionality, stochasticity and nonlinearity of biological gene regulation — a setting where standard trajectory models break down entirely.

\medskip
The most demanding evaluation is BM7 (LARRY), which tracks primary  cells differentiating into two terminal lineages (Figure~\ref{fig:larry_combined}). Here, we compare CADENCE directly against published results from dedicated, domain-specific trajectory inference methods evaluated on the same dataset. CADENCE discovers the terminal fates in a fully unsupervised manner, achieving state-of-the-art fate prediction accuracy (67.8\%, outperforming the best baseline at 58.5\%). Furthermore, it achieves the lowest population-level Sinkhorn divergence at both evaluation time points and superior fate probability calibration (AUROC 0.84), proving that CADENCE's cross-sectional training strategy scales effectively to primary, real-world dynamical systems.

\begin{figure}[t]
    \centering
    \begin{tikzpicture}
        \node[anchor=south west, inner sep=0] (img) at (0,0) {%
            \includegraphics[width=\linewidth]{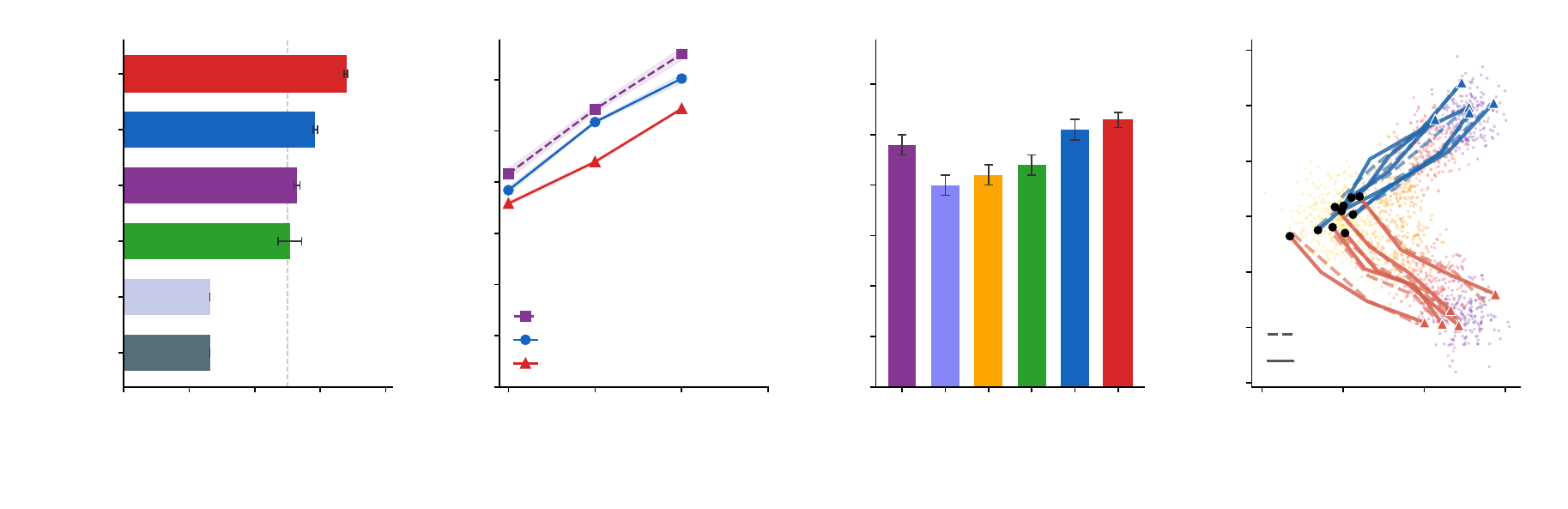}%
        };
        \begin{scope}[
            x={($(img.south east)-(img.south west)$)},
            y={($(img.north west)-(img.south west)$)}
        ]

        \node[anchor=center, inner sep=0] at (0.523, 0.575) {%
            \includegraphics[width=0.91\linewidth]{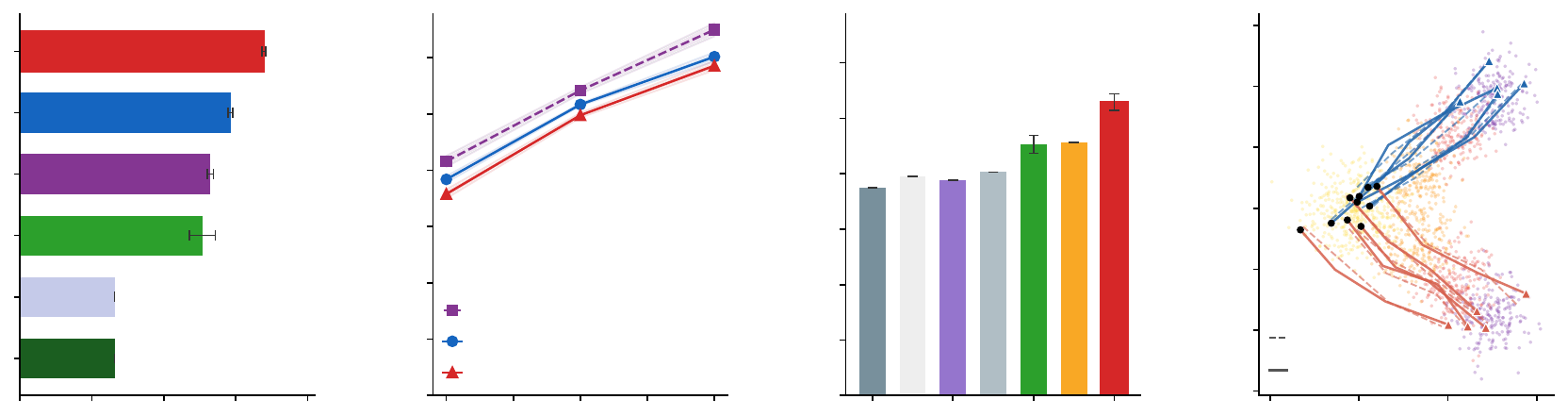}%
        };

        \node[anchor=north east, font=\small\bfseries] at (0.025,1.025) {(a)};
        \node[anchor=north west, font=\small\bfseries] at (0.265,1.025) {(b)};
        \node[anchor=north west, font=\small\bfseries] at (0.51,1.025) {(c)};
        \node[anchor=north west, font=\small\bfseries] at (0.745,1.025) {(d)};

        \node[anchor=north west, font=\small] at (0.09,1.025) {Fate accuracy};
        \node[anchor=north west, font=\small] at (0.31,1.025) {Sinkhorn per day};
        \node[anchor=north west, font=\small] at (0.553,1.025) {Fate prob.\ Pearson};
        \node[anchor=north west, font=\small] at (0.81,1.05) {Observation vs};
        \node[anchor=north west, font=\small] at (0.82,0.97) {CADENCE};

        \node[anchor=east, font=\scriptsize] at (0.08,0.854) {CADENCE};
        \node[anchor=east, font=\scriptsize] at (0.08,0.744) {scDiffEq};
        \node[anchor=east, font=\scriptsize] at (0.08,0.634) {PRESCIENT};
        \node[anchor=east, font=\scriptsize] at (0.08,0.524) {TIGON};
        \node[anchor=east, font=\scriptsize] at (0.08,0.414) {Dynamo};
        \node[anchor=east, font=\scriptsize] at (0.08,0.304) {CellRank};

        \node[anchor=west, font=\scriptsize] at (0.22,0.852) {67.8};
        \node[anchor=west, font=\scriptsize] at (0.199,0.742) {58.5};
        \node[anchor=west, font=\scriptsize] at (0.187,0.632) {52.9};
        \node[anchor=west, font=\scriptsize] at (0.188,0.522) {50.7};
        \node[anchor=west, font=\scriptsize] at (0.132,0.412) {26.3};
        \node[anchor=west, font=\scriptsize] at (0.132,0.302) {26.3};

        \node[anchor=north, font=\scriptsize] at (0.079,0.24) {0};
        \node[anchor=north, font=\scriptsize] at (0.121,0.24) {20};
        \node[anchor=north, font=\scriptsize] at (0.163,0.24) {40};
        \node[anchor=north, font=\scriptsize] at (0.204,0.24) {60};
        \node[anchor=north, font=\scriptsize] at (0.246,0.24) {80};
        \node[anchor=north, font=\scriptsize] at (0.165,0.185) {Fate prediction accuracy (\%)};

        \node[rotate=90, anchor=center, font=\scriptsize] at (0.28,0.580) {Sinkhorn divergence};
        \node[anchor=east, font=\scriptsize] at (0.32,0.238) {0};
        \node[anchor=east, font=\scriptsize] at (0.32,0.338) {5};
        \node[anchor=east, font=\scriptsize] at (0.32,0.439) {10};
        \node[anchor=east, font=\scriptsize] at (0.32,0.540) {15};
        \node[anchor=east, font=\scriptsize] at (0.32,0.641) {20};
        \node[anchor=east, font=\scriptsize] at (0.32,0.741) {25};
        \node[anchor=east, font=\scriptsize] at (0.32,0.842) {30};

        \node[anchor=north, font=\scriptsize] at (0.324,0.24) {Day 3};
        \node[anchor=north, font=\scriptsize] at (0.405,0.24) {Day 4};
        \node[anchor=north, font=\scriptsize] at (0.48,0.24) {Day 5};
        \node[anchor=north, font=\scriptsize] at (0.405,0.185) {Time point};

        \definecolor{prescientcol}{HTML}{9467bd}
        \definecolor{scdiffeqcol}{HTML}{1f77b4}
        \definecolor{cadencecol}{HTML}{d62728}
        \node[anchor=west, font=\fontsize{6}{7}\selectfont] at (0.335,0.393) {PRESCIENT};
        \node[anchor=west, font=\fontsize{6}{7}\selectfont] at (0.335,0.337) {scDiffEq};
        \node[anchor=west, font=\fontsize{6}{7}\selectfont] at (0.335,0.282) {CADENCE};

        \node[rotate=90, anchor=center, font=\scriptsize] at (0.515,0.580) {Pearson correlation};
        \node[anchor=east, font=\scriptsize] at (0.561,0.238) {0.0};
        \node[anchor=east, font=\scriptsize] at (0.561,0.337) {0.1};
        \node[anchor=east, font=\scriptsize] at (0.561,0.436) {0.2};
        \node[anchor=east, font=\scriptsize] at (0.561,0.535) {0.3};
        \node[anchor=east, font=\scriptsize] at (0.561,0.634) {0.4};
        \node[anchor=east, font=\scriptsize] at (0.561,0.734) {0.5};
        \node[anchor=east, font=\scriptsize] at (0.561,0.833) {0.6};

        \node[rotate=35, anchor=east, font=\fontsize{6}{7}\selectfont] at (0.585,0.215) {PBA};
        \node[rotate=35, anchor=east, font=\fontsize{6}{7}\selectfont] at (0.610,0.215) {fateID};
        \node[rotate=35, anchor=east, font=\fontsize{6}{7}\selectfont] at (0.635,0.215) {TrajectoryNet};
        \node[rotate=35, anchor=east, font=\fontsize{6}{7}\selectfont] at (0.66,0.215) {WOT};
        \node[rotate=35, anchor=east, font=\fontsize{6}{7}\selectfont] at (0.686,0.215) {TIGON};
        \node[rotate=35, anchor=east, font=\fontsize{6}{7}\selectfont] at (0.71,0.215) {MIOFlow};
        \node[rotate=35, anchor=east, font=\fontsize{6}{7}\selectfont] at (0.736,0.215) {CADENCE};

        \node[rotate=90, anchor=west, font=\fontsize{6}{7}\selectfont] at (0.574,0.600) {0.375};
        \node[rotate=90, anchor=west, font=\fontsize{6}{7}\selectfont] at (0.598,0.620) {0.395};
        \node[rotate=90, anchor=west, font=\fontsize{6}{7}\selectfont] at (0.621,0.610) {0.388};
        \node[rotate=90, anchor=west, font=\fontsize{6}{7}\selectfont] at (0.645,0.635) {0.403};
        \node[rotate=90, anchor=west, font=\fontsize{6}{7}\selectfont] at (0.668,0.69) {0.453};
        \node[rotate=90, anchor=west, font=\fontsize{6}{7}\selectfont] at (0.692,0.69) {0.456};
        \node[rotate=90, anchor=west, font=\fontsize{6}{7}\selectfont] at (0.715,0.765) {0.529};

        \node[rotate=90, anchor=center, font=\scriptsize] at (0.76,0.580) {PC 2};
        \node[anchor=east, font=\scriptsize] at (0.8,0.245) {$-3$};
        \node[anchor=east, font=\scriptsize] at (0.8,0.354) {$-2$};
        \node[anchor=east, font=\scriptsize] at (0.8,0.463) {$-1$};
        \node[anchor=east, font=\scriptsize] at (0.8,0.572) {0};
        \node[anchor=east, font=\scriptsize] at (0.8,0.682) {1};
        \node[anchor=east, font=\scriptsize] at (0.8,0.791) {2};
        \node[anchor=east, font=\scriptsize] at (0.8,0.900) {3};

        \node[anchor=north, font=\scriptsize] at (0.806,0.24) {$-2$};
        \node[anchor=north, font=\scriptsize] at (0.858,0.24) {0};
        \node[anchor=north, font=\scriptsize] at (0.909,0.24) {2};
        \node[anchor=north, font=\scriptsize] at (0.961,0.24) {4};
        \node[anchor=north, font=\scriptsize] at (0.885,0.185) {PC 1};

        \node[anchor=west, font=\fontsize{6}{7}\selectfont] at (0.812,0.344) {Observed};
        \node[anchor=west, font=\fontsize{6}{7}\selectfont] at (0.812,0.287) {CADENCE};

        \end{scope}
    \end{tikzpicture}
    \caption{BM7 (LARRY haematopoiesis) results comparing CADENCE against published baselines.
    \textbf{(a)}~Fate prediction accuracy: CADENCE discovers Neutrophil vs.\ Monocyte fates unsupervised, outperforming all baselines including scDiffEq (67.8\% vs.\ 58.5\%).
    \textbf{(b)}~Per-day Sinkhorn divergence: CADENCE achieves the lowest population-level distributional distance at all evaluation time points.
    \textbf{(d)}~Fate probability Pearson correlation against clonal ground truth.
    \textbf{(d)}~Comparison of observed and CADENCE predicted clonal evolution trajectories.
    }
    \label{fig:larry_combined}
\end{figure}
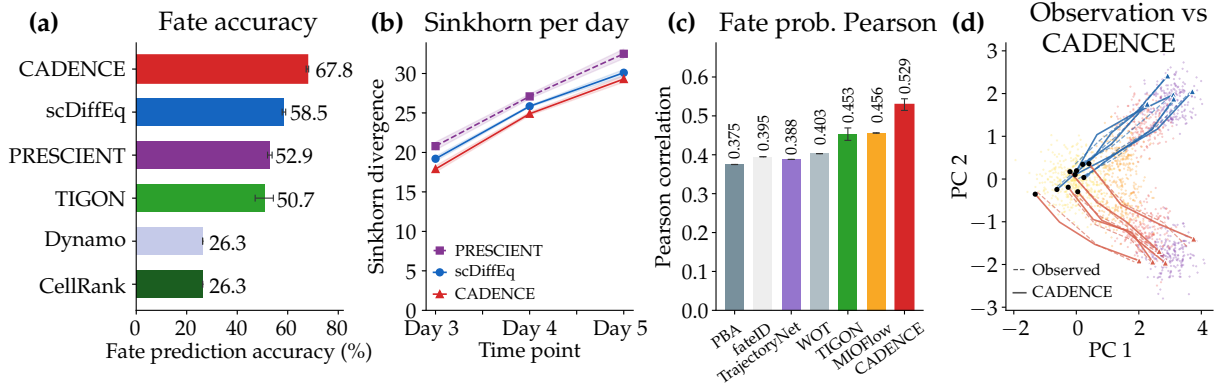

\paragraph{Ablation Study.} To isolate the contributions of our architectural choices, we conduct a systematic ablation study. 
First, comparing four Stage 2 dynamics variants, 
the Augmented Neural ODE consistently achieves the best performance across benchmarks, outperforming standard ODE, Multi-Block, and Flow Matching architectures. Second, for high-dimensional observations ($p \gg q$), ablating spatial compression strategies 
demonstrates that mapping to a compact latent space via PCA or VAE is essential for stable dynamics learning, drastically outperforming uncompressed linear baselines. Finally, a controlled component ablation 
probes the model's sensitivity to sample size, observation noise, and network capacity. Crucially, disabling the SMoE router triggers severe performance degradation, validating that our context-driven routing mechanism is indispensable for recovering heterogeneous dynamics.

\bigskip
\section{Discussion and Limitations}
\label{sec:discussion}

We introduced CADENCE, a model class that breaks a longstanding dichotomy between aggregate mass transport and data-hungry sequential latent models. By showing how sparse cross-sectional snapshots can be rigorously converted into effective longitudinal signals, CADENCE provides new identifiability guarantees for trajectory inference from heterogeneous data. The CADENCE architecture is a minimal  implementation of key geometric/dynamical constraints (rather than only a collection of heuristic design choices).

\paragraph{Limitations.} The validity of the recovered trajectories rests on these structural assumptions. The most fundamental vulnerability is unobserved confounding: if critical information dictating the dynamical archetype is missing from the static context, COA is violated. 
In this regime, the model risks losing individual identifiability, gracefully degrading into an ensemble-level optimal transport flow rather than capturing true individual dynamics. 
That said, in settings with rich context information, it is arguably reasonable to expect information as needed for COA to be available in the static context vector. Furthermore, fixing the expert pool $K$ may yield high MRA approximation errors in vastly heterogeneous systems, though the Carath\'{e}odory bound ($K \ge d'+1$) provides a mathematically rigorous lower bound for empirical ablations. 

\paragraph{Broader Impacts and Future Directions.} CADENCE unlocks individualized, continuous-time forecasting in domains where full longitudinal tracking is difficult or impossible. To scale this paradigm, overcoming optimization stiffness in the strictly cross-sectional limit remains an exciting frontier. Future extensions could treat the expert pool as a non-parametric Bayesian prior (e.g., a Dirichlet Process) to dynamically discover an unbounded number of dynamical archetypes. Finally, as with any model conditioned on  contexts, predictive vector fields risk propagating biases into forecasted futures; rigorous, domain-specific validation remains imperative prior to consequential deployment.

\vspace{1cm}
{\bf Acknowledgments.}
The authors gratefully acknowledge the Gauss Centre for Supercomputing e.V. for supporting this project by providing computing time on the GCS Supercomputers.
\bibliographystyle{abbrvnat}
\bibliography{refs}

\onecolumn
\clearpage

\appendix

\renewcommand{\thesection}{\Alph{section}}

\renewcommand{\thesubsection}{\thesection.\arabic{subsection}}

\renewcommand{\thefigure}{\thesection.\arabic{figure}}

\renewcommand{\thetable}{\thesection.\arabic{table}}

\counterwithin{figure}{section}
\counterwithin{table}{section}

{\Large {\bf Appendices.}}

\section{Proofs of Theoretical Results}
\label{app:proofs}

\subsection{Proof of Proposition~\ref{prop:nonid} (Non-identifiability without structure)}
Let $\phi:\R^q\to\R^q$ be a smooth diffeomorphism. Define $\tilde{\bz}^i_t=\phi(\bz^i_t)$. By the chain rule, $$\tfrac{d\tilde{\bz}^i_t}{dt} = D\phi(\phi^{-1}(\tilde{\bz}^i_t))h_\theta(\phi^{-1}(\tilde{\bz}^i_t),t,\bc_i;\bw_i) = \tilde{h}(\tilde{\bz}^i_t,t,\bc_i;\bw_i),$$ so $\tilde{\bz}^i_t$ satisfies the ODE with $\tilde{h}$ and initial condition $\phi(\bz^i_0)$. Reconstruction is unchanged: $$\mathcal{F}_\psi^{-1}(\phi^{-1}(\phi(\bz^i_{t_i})))=\bx^i_{t_i}.$$ Since $\mathrm{Diff}(\R^q)$ is infinite-dimensional, valid configurations are uncountably infinite. \hfill$\square$

\subsection{Proof of Proposition~\ref{prop:reduced_ambiguity} (Reduced Spatial Ambiguity)}
If both $(\mathcal{F}_\psi)_\#p(x)=\mathcal{N}(\mathbf{0},\mathbf{I})$ and $(\tilde{\mathcal{F}}_\psi)_\#p(x)=\mathcal{N}(\mathbf{0},\mathbf{I})$, then $\varphi=\tilde{\mathcal{F}}_\psi\circ\mathcal{F}_\psi^{-1}$ satisfies $\varphi_\#\mathcal{N}(\mathbf{0},\mathbf{I})=\mathcal{N}(\mathbf{0},\mathbf{I})$, so $\varphi\in\mathcal{S}_{\mathrm{iso}}$. The PF-ODE bijection is unique given a fixed score function \cite{song2021score}. 

\medskip
\emph{Remark on $\mathcal{S}_{\mathrm{iso}}$ vs.\ $O(q)$:} The subgroup $\mathcal{S}_\text{iso}$ of $\mathcal{N}(0,\mathbf{I})$-preserving diffeomorphisms strictly contains the orthogonal group $O(q)$: the twist map $(r, \theta) \mapsto (r, \theta + f(r))$ in polar coordinates preserves the Gaussian radial density for any smooth $f$ but is non-orthogonal. We therefore do not claim the residual symmetry reduces to $O(q)$ alone. Theorem~\ref{thm:composition} shows these residual symmetries are observationally harmless: the prediction $\hat{x}^i_t$ is invariant under conjugation by any $\varphi \in \mathcal{S}_\text{iso}$. \hfill$\square$

\subsection{Formal Regularity Conditions (Assumptions~\ref{asm:pr} \& \ref{asm:foa})}
\label{app:foa_conditions}
\begin{assumption}[Parametric Restriction, PR]\label{asm:pr}
  $h_\theta$ belongs to a parametric family $\mathcal{H}\!=\!\{h_\theta:\theta\in\Theta\}$ with $\Theta\subset\R^m$ finite-dimensional, where distinct parameters produce distinct flows.
\end{assumption}

\begin{assumption}[Flow Observability, FOA]\label{asm:foa}
  For each leaf $\lambda$, the map $w\mapsto\{(\Phi^w_{t,t_0})_\#\rho^\lambda_0:t\in\mathcal{T}_\lambda\}$ is injective within $\mathcal{H}$: distinct parameters produce observationally distinct families of time-indexed pushforward marginals.
\end{assumption}

Under Assumption~\ref{asm:foa}, spatial richness converts pointwise flow divergence into distributional distinguishability. Specifically, we delineate the regimes where FOA holds versus where it breaks down:

\paragraph{When FOA holds (distributional constraints are rich enough):}
\
\begin{itemize}[leftmargin=*,itemsep=2pt]
    \item \emph{Spatial richness:} The baseline distribution $\rho^\lambda_0$ has density bounded below on an open set $V\subseteq\R^q$, allowing pushforward measures to probe the vector field across a positive-measure region.
    \item \emph{Temporal richness:} Enrollment times span an interval $[t_{\min},t_{\max}]$ wide enough that $$\|(\Phi^w_{t,t_0})_\#\rho^\lambda_0 - (\Phi^{w'}_{t,t_0})_\#\rho^\lambda_0\|_{\mathrm{TV}}>0$$ for some $t$ and all distinct parameters $w\neq w'$ in $\mathcal{H}$.
\end{itemize}

\paragraph{When FOA fails (distributional constraints are insufficient):}
\begin{itemize}[leftmargin=*,itemsep=2pt]
    \item Observations concentrate near a fixed point of the dynamics, where distinct vector fields yield indistinguishable pushforward measures.
    \item Enrollment times collapse to a single point, entirely eliminating temporal variation.
    \item The baseline $\rho^\lambda_0$ concentrates precisely on a "parameter confusion zone" where distinct parameters produce identical flows. (We note that the BM7 experiment deliberately places the model in this marginal FOA regime to rigorously test failure bounds.)
\end{itemize}

\subsection{Proof of Theorem~\ref{thm:foliation_id} (Parameter Identifiability Given Observable Foliation)}

\textbf{Step 1 (Latent recovery).} $\mathcal{F}_\psi$ is a smooth bijection (condition~iii), so $z^i_{t_i}=\mathcal{F}_\psi(x^i_{t_i};c_i)$ is uniquely recoverable. The measure $Q$ on $\R^q\times\mathcal{C}\times\mathcal{T}$ is uniquely determined.\\

\textbf{Step 2 (Leaf assignment).} By COA (Assumption~\ref{asm:coa}), $g^*:\mathcal{C}\to\Lambda$ is measurable, deterministic function. Thus the leaf assignment $\lambda(i)=g^*(c_i)$ is uniquely determined by the context alone; no observations $x^i_{t_i}$ are used. The foliation is assumed observable via COA, not identified from data.\\

\textbf{Step 3 (Within-leaf parameter recovery --- non-trivial).} Within leaf $\lambda$, data provides samples from $P(z_t|\lambda,t)=(\Phi^{\hat{w}(\lambda)}_{t,t_0})_\#\rho^\lambda_0$ at enrollment times $t\in\mathcal{T}_\lambda$. By FOA (Assumption~\ref{asm:foa}), the map $$\hat{w}\mapsto\{(\Phi^{\hat{w}}_{t,t_0})_\#\rho^\lambda_0:t\in\mathcal{T}_\lambda\}$$ is injective within $\mathcal{H}$. Condition~(ii) ensures enrollment times span an interval, so distributional constraints are non-degenerate. Injectivity implies $\hat{w}(\lambda)$ is uniquely determined---exactly identifiable. The $\epsilon$-gap to $w(\lambda)$ is bounded by MRA (Assumption~\ref{asm:mra}): $\|w(\lambda)-\hat{w}(\lambda)\|\leq\epsilon$. This step requires no individual trajectory pairs; it operates entirely through distributional constraints. \emph{Why PR (Assumption~\ref{asm:pr}) is essential:} without restricting $h_\theta$ to $\mathcal{H}$, the inverse problem of recovering a vector field from pushforward marginals is ill-posed for nonparametric families.\\

\textbf{Step 4 (Cross-leaf distinguishability).} By condition~(i), $\lambda\neq\lambda'$ implies $w(\lambda)\neq w(\lambda')$ within $\mathcal{H}$. Combined with Step 2, the system $(g^*,\{w(\lambda)\})$ is identified up to leaf-label permutation. \hfill$\square$

\subsection{Extended Basis Identifiability via Simplex Geometry}
\label{app:extended_identifiability}

\begin{lemma}[Basis identifiability under anchor conditions]
\label{lem:1b}
Under Theorem~\ref{thm:foliation_id}, if additionally (i) for each $k\in[K]$ there exists $\lambda_k$ with $\alpha(\lambda_k)=e_k$ (anchor leaves), and (ii) $\{w_1,\ldots,w_K\}$ are affinely independent, then $\bW=[w_1,\ldots,w_K]$ is identifiable up to column permutation.
\end{lemma}
\begin{proof}
At anchor leaves $\hat{w}(\lambda_k)=\bW e_k=w_k$ directly recovers each column. Affine independence ensures distinctness; routing is recovered as $\alpha(\lambda)=\bW^\dagger\hat{w}(\lambda)$.
\end{proof}

\begin{lemma}[Basis identifiability via minimum-volume simplex]
\label{lem:1c}
If $\hat{w}(\Lambda)$ has nonempty relative interior in the affine hull of $\mathrm{conv}(\bW)$, and the minimum-volume simplex containing $\hat{w}(\Lambda)$ is unique, then $\bW$ is identifiable up to column permutation without anchor leaves.
\end{lemma}
\begin{proof}[Proof sketch]
By construction $\hat{w}(\Lambda)\subseteq\mathrm{conv}(\bW)$. Non-empty relative interior implies $\mathrm{conv}(\bW)$ is the unique minimum-volume simplex containing $\hat{w}(\Lambda)$ (Craig-type separability condition from blind source separation). Its vertices are the columns of $\bW$ up to permutation.
\end{proof}

These lemmas provide the explicit theoretical foundation for our algorithmic design. Specifically, Lemma~\ref{lem:1b} motivates our temperature annealing strategy (see Proposition~\ref{prop:anneal}): 
as the temperature $\tau \to 0$, routing coefficients approach one-hot vectors, dynamically creating the exact "anchor leaves" necessary to stabilize basis recovery. Conversely, Lemma~\ref{lem:1c} governs the soft-routing regime (moderate $\tau$ early in training). It guarantees identifiability even before anchor leaves form, provided the routing coefficients span the interior of the minimum-volume simplex. This simplex interpretation directly connects our architecture to Craig-type separability in the hyperspectral unmixing literature, providing an efficient algorithmic mechanism for basis recovery without requiring explicit anchor identification.

\subsection{Proof of Theorem~\ref{thm:composition} (Composition Identifiability in Quotient Space)}

\textbf{Step 1.} By Proposition~\ref{prop:reduced_ambiguity}, PF-ODE training restricts symmetry to $\mathcal{S}_{\mathrm{iso}}$. Define: $(\mathcal{F}_\psi,h_\theta,g^*)\sim(\tilde{\mathcal{F}}_\psi,\tilde{h}_\theta,g^*)$ if $\tilde{\mathcal{F}}_\psi=\varphi\circ\mathcal{F}_\psi$ for some $\varphi\in\mathcal{S}_{\mathrm{iso}}$ and $\tilde{h}$ is the conjugated vector field.\\

\textbf{Step 2.} By Theorem~\ref{thm:foliation_id}, $\{w(\lambda)\}$ and $g^*$ are identifiable within any fixed frame; both are invariant within equivalence classes ($g^*$ depends only on $c_i$; $w(\lambda)$ lives in parameter space, unaffected by $\varphi$). Note that the transformation $\varphi \in \mathcal{S}_\text{iso}$ acts on $(F_\psi, h_\theta)$ by conjugation but leaves $(g^*, w(\lambda))$ invariant by construction: the routing function $g^*$ depends only on context $c_i$ (which lives in covariate space, not latent state space), and the dynamical parameters $w(\lambda)$ parameterize the ODE in the original coordinate frame, not in the transformed frame.\\

\textbf{Step 3 ($\varphi$ cancels).} The observable prediction is $\hat{x}^i_t=\mathcal{F}_\psi^{-1}(\Phi^{w(\lambda(i))}_t(z^i_0))$. Under the equivalent model $\tilde{\mathcal{F}}_\psi=\varphi\circ\mathcal{F}_\psi$: encoded baseline is $\varphi(z^i_0)$, conjugated flow is $\tilde{\Phi}^w=\varphi\circ\Phi^w\circ\varphi^{-1}$, and the prediction is $(\varphi\circ\mathcal{F}_\psi)^{-1}(\varphi\circ\Phi^w(z^i_0))=\mathcal{F}_\psi^{-1}(\Phi^w_t(z^i_0))=\hat{x}^i_t$. The $\varphi$ terms cancel exactly.\\

\textbf{Step 4.} Observable equivalence within classes and identifiability within classes (Step 2) establish that $[\mathcal{F}_\psi,h_\theta,g^*]\in\mathcal{M}/\mathcal{S}_{\mathrm{iso}}$ is uniquely determined by the data. \hfill$\square$

\subsection{Conditional Pooled Identifiability}

\begin{corollary}[Conditional Pooled Identifiability under DFA]
\label{cor:dfa_pooled}
  Under Assumptions~\ref{asm:dfa}, \ref{asm:pr}, \ref{asm:foa}, if enrollment times span $[t_{\min},t_{\max}]$ per leaf, $\rho^\lambda_0$ has open-set support, and $\mathcal{F}_\psi$ is injective, then $\hat{w}(\lambda)$ is identifiable within $\mathcal{H}$ up to $\mathcal{S}_{\mathrm{iso}}$-symmetry.
\end{corollary}

\noindent \textit{Proof.} Within leaf $\lambda$, the Liouville equation governs density evolution: $\partial_t\rho^\lambda_t+\nabla_z\cdot(\rho^\lambda_t\cdot h_\theta(\cdot;w(\lambda)))=0$, with $\rho^\lambda_t=(\Phi^{w(\lambda)}_{t,t_0})_\#\rho^\lambda_0$. Cross-sectional data provides samples from $P(z_t|\lambda,t)=\rho^\lambda_t$ at enrollment times $t\in[t_{\min},t_{\max}]$. The open-support condition on $\rho^\lambda_0$ and span of enrollment times jointly satisfy the FOA sufficient conditions (Appendix~\ref{app:foa_conditions}): pushforward measures are constrained at multiple times with sufficient spatial richness. FOA injectivity recovers $\hat{w}(\lambda)$ within $\mathcal{H}$ up to $\mathcal{S}_{\mathrm{iso}}$-symmetry. \hfill$\square$

\vspace{0.5em}
\noindent \textbf{Remark (Why PR is essential).} Without restricting to a parametric family $\mathcal{H}$, Corollary ~\ref{cor:dfa_pooled} would be false: many nonparametric vector fields produce identical marginal distributions from cross-sectional samples (the inverse transport problem from marginals alone is known to be ill-posed without regularization). PR is the crucial ingredient that makes the inverse problem well-posed — it should be read not as a limitation but as a transparent statement of what cross-sectional data can and cannot recover.

\subsection{Temperature Annealing Convergence}

\begin{proposition}[Temperature Annealing]
\label{prop:anneal}
  Let $\balpha_i(\tau)\!=\!\Softmax(g_\omega(\bc_i)/\tau)$ for $\tau\!>\!0$. As $\tau\!\to\!0^+$, $\balpha_i(\tau)\!\to\!\be_{\kappa(i)}$, recovering hard leaf assignments of the discrete-leaf DFA.
\end{proposition}

\noindent \textit{Proof.} For $v\in\R^K$ with unique maximiser $k^*$, $$[\Softmax(v/\tau)]_k=(1+\sum_{j\neq k}e^{-(v_k-v_j)/\tau})^{-1}\to\mathbf{1}[k=k^*]$$ as $\tau\to 0^+$, since $e^{-\Delta/\tau}\to 0$ for $\Delta>0$. Holds for Lebesgue-a.e.\ $v$ (ties have measure zero). \hfill$\square$

\subsection{From the Cross-Sectional MMD Objective to FOA Identifiability}
\label{app:mmd_foa_bridge}

In Section 5, we established the exact identifiability of CADENCE under the continuous, population-level Flow Observability Assumption (FOA, Assumption~\ref{asm:foa}). In practice, CADENCE approximates this theoretical ideal via finite-sample empirical risk minimization, matching a temporally smoothed empirical reference using a Maximum Mean Discrepancy (MMD) objective (Eq.~\ref{eq:cross_loss}). This subsection formally bridges these two perspectives, demonstrating how our finite-sample MMD objective serves as a statistically consistent estimator of the FOA-identified parameters.

We structure this consistency argument in two modular stages. This deliberate decoupling provides a clean theoretical separation between the estimation of the reference measure and the recovery of the routing parameters:
\begin{enumerate}[leftmargin=*,topsep=2pt,itemsep=2pt]
  \item \textbf{Leaf-level argmin consistency} (Proposition~\ref{prop:leaf_mmd_foa}): \emph{Given} a consistently estimated leaf assignment, the per-leaf MMD argmin is a consistent estimator of the FOA-faithful parameters $w^\star(\lambda)$. This follows from standard M-estimation under a quantitative FOA condition.
  \item \textbf{Dictionary recovery} (Corollary~\ref{cor:dictionary_recovery}): Given consistent leaf-level estimates $\hat w_N(\lambda)\to w^\star(\lambda)$, the basis $\bW^\star$ and routing $\balpha^\star(\lambda)$ are identified up to column permutation. We delineate two distinct convergence regimes governed by the annealing schedule: an \emph{anchor regime} ($\tau_N\to 0^+$ for one-hot routing) and a \emph{simplex-interior regime} ($\tau_N\to\tau_0>0$ to preserve soft routing).
\end{enumerate}
By decoupling these steps, we  bypass the complex uniform convergence bounds that would otherwise be required if routing weights and reference measures were analyzed simultaneously. We discuss the standalone consistency of the leaf assignment $\hat\lambda_N$ (e.g., via the COA-implied context-to-leaf map) at the end of this subsection.

\paragraph{Setup.} Let $N$ index units in the cross-sectional ensemble and $N_\lambda$ denote the number of units in leaf $\lambda$. Denote by $h_N>0$ the bandwidth of the temporal Gaussian kernel $\kappa_{h_N}:\R\to\R$; the smoothing acts on the scalar enrollment-time axis only, so $d_\mathcal{T}=1$. Let $\tau_N$ denote the routing softmax temperature. Assume the temporal observation interval $\mathcal{T}\subset\R$ is compact and the parameter space $\Theta$ is compact, with $h_\theta\in\mathcal{H}=\{h_\theta(\,\cdot\,;w):w\in\Theta\}$ (PR). Let $k_\sigma:\R^q\times\R^q\to\R$ be a bounded translation-invariant \emph{characteristic} kernel \cite{sriperumbudur2010hilbert} (e.g., Gaussian with fixed spatial bandwidth $\sigma$); $\mathrm{MMD}_{k_\sigma}(\mu,\nu)=0$ iff $\mu=\nu$. We write $\rho^\lambda_t := (\Phi^{w^\star(\lambda)}_{t,t_0})_\# \rho^\lambda_0$ for the FOA-faithful pushforward of leaf $\lambda$ at time $t$ and assume the joint density $(t,z)\mapsto p^\lambda(z\mid t)$ of $\rho^\lambda_t$ is continuous on $\mathcal{T}^\circ\times\R^q$; this naturally follows from the continuity of the baseline density $\rho^\lambda_0$ and the smoothness of the ODE flow under PR.

\medskip
For the bandwidth: the standard univariate-KDE consistency rate is $h_N\to 0$, $Nh_N\to\infty$. A common explicit choice is $h_N=N^{-a}$ for any $a\in(0,1)$; under twice-differentiable temporal densities, the MSE-optimal univariate-KDE rate is $h_N \asymp N^{-1/5}$. Boundary effects are avoided by restricting integration to $\mathcal{T}^\circ$.

\paragraph{Quantitative FOA.} The standard FOA (Assumption~\ref{asm:foa}) provides \emph{population-level} injectivity: distinct parameters yield distinct pushforward families. To operationalize this for statistical argmin-consistency, we define a stable, \emph{Quantitative FOA} (QFOA) condition. Under our setup ($\Theta$ compact, $w\mapsto(\Phi^w_{t,t_0})_\#\rho^\lambda_0$ continuous in $\mathrm{MMD}_{k_\sigma}$ uniformly in $t\in\mathcal{T}_\lambda$), FOA strictly implies QFOA:
\begin{equation}
  \label{eq:qfoa}
  \forall\,\eta>0,\quad
  \inf_{w\in\Theta:\,\|w-w^\star(\lambda)\|\ge\eta}\;
  \int_{\mathcal{T}_\lambda}\!\mathrm{MMD}_{k_\sigma}^2\!\left((\Phi^w_{t,t_0})_\#\rho^\lambda_0,\;\rho^\lambda_t\right)\,\mathrm{d}\nu_\lambda(t)\;>\;0.
\end{equation}
This holds because the integrand is continuous in $w$, vanishes only at $w=w^\star(\lambda)$ (due to FOA and the characteristic kernel), and the infimum over the compact set $\{w:\|w-w^\star(\lambda)\|\ge\eta\}\subset\Theta$ is attained and strictly positive. Equation~\eqref{eq:qfoa} serves as the well-separated-minimum condition required for M-estimation.

\medskip

\begin{proposition}[Leaf-level MMD-to-FOA argmin consistency]
\label{prop:leaf_mmd_foa}
Fix a leaf $\lambda$. Suppose the per-unit leaf assignment is either known or estimated by $\hat\lambda_N$ with $\Pr(\hat\lambda_N(i)\neq\lambda(i))\to 0$ as $N\to\infty$. Define the leaf-level temporally smoothed empirical reference
\[
  \hat\rho^\lambda_{t,N}\;\propto\;\sum_{j:\hat\lambda_N(j)=\lambda} \kappa_{h_N}(t_j-t)\,\delta_{\bz^j_{t_j}},
\]
the population leaf-level risk
\[
  L_\lambda(w)\;:=\;\int_{\mathcal{T}_\lambda}\mathrm{MMD}_{k_\sigma}^2\!\left((\Phi^w_{t,t_0})_\#\rho^\lambda_0,\;\rho^\lambda_t\right)\,\mathrm{d}\nu_\lambda(t),
\]
and the empirical leaf-level risk
\[
  L_{\lambda,N}(w)\;:=\;\int_{\mathcal{T}_\lambda}\mathrm{MMD}_{k_\sigma}^2\!\left((\Phi^w_{t,t_0})_\#\rho^\lambda_0,\;\hat\rho^\lambda_{t,N}\right)\,\mathrm{d}\nu_\lambda(t).
\]
Assume:
\begin{enumerate}[leftmargin=*,topsep=2pt,itemsep=2pt]
  \item \emph{Compactness and continuity:} $\Theta$ is compact and $w\mapsto(\Phi^w_{t,t_0})_\#\rho^\lambda_0$ is continuous in $\mathrm{MMD}_{k_\sigma}$, uniformly in $t\in\mathcal{T}_\lambda$. (Holds whenever the ODE flow is jointly continuous in $(w,t,z)$ and the baseline density is continuous.)
  \item \emph{Bounded characteristic kernel:} $k_\sigma$ is bounded and characteristic.
  \item \emph{Sampling:} enrollment times $\{t_j\}$ are i.i.d.\ from a density $f_T$ continuous and bounded below on $\mathcal{T}_\lambda$, and $\rho^\lambda_0$ has a continuous Lebesgue density bounded below on an open set $V\subseteq\R^q$.
  \item \emph{Bandwidth:} $h_N\to 0$, $N_\lambda h_N\to\infty$.
  \item \emph{Quantitative FOA} \eqref{eq:qfoa}: $w^\star(\lambda)$ is a well-separated minimum of $L_\lambda$.
  \item \emph{Reference consistency:} $\int_{\mathcal{T}_\lambda}\mathrm{MMD}_{k_\sigma}^2(\hat\rho^\lambda_{t,N},\,\rho^\lambda_t)\,\mathrm{d}\nu_\lambda(t)\xrightarrow{\mathrm{P}} 0$ (Lemma~\ref{lem:temporal_kde} below).
\end{enumerate}
Then the leaf-level estimator $\hat w_N(\lambda)\;\in\;\argmin_{w\in\Theta}\,L_{\lambda,N}(w)$ satisfies $\hat w_N(\lambda)\xrightarrow{\mathrm{P}} w^\star(\lambda)$ as $N\to\infty$.
\end{proposition}

\noindent\textit{Proof.} This follows from standard M-estimation arguments \cite[Theorem~5.7]{vandervaart1998asymptotic}. By assumption~(2), characteristic-kernel identifiability, and FOA, $L_\lambda(w)=0\iff w=w^\star(\lambda)$. Quantitative FOA \eqref{eq:qfoa} ensures that $w^\star(\lambda)$ is a well-separated minimum. Reference consistency (assumption~6, Lemma~\ref{lem:temporal_kde}) combined with the uniform continuity of $w\mapsto(\Phi^w_{t,t_0})_\#\rho^\lambda_0$ on the compact set $\Theta$ (assumption~1) yields uniform convergence of the empirical risk: $\sup_{w\in\Theta}|L_{\lambda,N}(w)-L_\lambda(w)|\xrightarrow{\mathrm{P}} 0$. Specifically, $|L_{\lambda,N}(w)-L_\lambda(w)|$ is bounded by a constant (from the bounded MMD kernel) times the integrated MMD distance between $\hat\rho^\lambda_{t,N}$ and $\rho^\lambda_t$, plus a uniform-continuity remainder controlled by the compactness of $\Theta$. The argmin-consistency theorem then guarantees $\hat w_N(\lambda)\xrightarrow{\mathrm{P}} w^\star(\lambda)$. \hfill$\square$

\medskip

\begin{lemma}[Integrated-risk consistency of the temporal-KDE reference]
\label{lem:temporal_kde}
Under the sampling and bandwidth conditions of Proposition~\ref{prop:leaf_mmd_foa}, with $k_\sigma$ a bounded characteristic kernel and $\rho^\lambda_t$ having a jointly continuous density $(t,z)\mapsto p^\lambda(z\mid t)$,
\[
  \int_{\mathcal{T}_\lambda}\mathrm{MMD}_{k_\sigma}^2\!\left(\hat\rho^\lambda_{t,N},\,\rho^\lambda_t\right)\,\mathrm{d}\nu_\lambda(t)\;\xrightarrow{\mathrm{P}}\; 0.
\]
\end{lemma}

\noindent\textit{Proof sketch.} Pointwise: for any bounded continuous $f$, $$\E[\int f\,\mathrm{d}\hat\rho^\lambda_{t,N}] = \E[f(Z_1)\kappa_{h_N}(t_1-t)]/\E[\kappa_{h_N}(t_1-t)].$$ The numerator equals $\int_{\mathcal{T}_\lambda}\int f(z)\,p^\lambda(z\mid s)\,f_T(s)\,\kappa_{h_N}(s-t)\,\mathrm{d}z\,\mathrm{d}s$, which converges to $f_T(t)\int f\,\mathrm{d}\rho^\lambda_t$ as $h_N\to 0$ by the joint continuity of $p^\lambda$; the denominator converges to $f_T(t)$. The variance is $\mathcal{O}(1/(N_\lambda h_N))\to 0$. Slutsky's theorem gives weak convergence in probability $\hat\rho^\lambda_{t,N}\Rightarrow\rho^\lambda_t$ for each $t\in\mathcal{T}_\lambda^\circ$. Because MMD with a bounded characteristic kernel metrizes weak convergence \cite[Thm.~5]{sriperumbudur2010hilbert}, and the kernel is bounded ($\mathrm{MMD}_{k_\sigma}^2\le 4\|k_\sigma\|_\infty$), we obtain pointwise convergence $\mathrm{MMD}_{k_\sigma}^2(\hat\rho^\lambda_{t,N},\rho^\lambda_t)\xrightarrow{\mathrm{P}} 0$. Dominated convergence on the compact set $\mathcal{T}_\lambda$ then upgrades this pointwise convergence to integrated convergence. \hfill$\square$

\medskip

\begin{corollary}[Dictionary and routing recovery]
\label{cor:dictionary_recovery}
Suppose Proposition~\ref{prop:leaf_mmd_foa} applies for every leaf $\lambda\in\Lambda$, ensuring $\hat w_N(\lambda)\xrightarrow{\mathrm{P}} w^\star(\lambda)$ uniformly over $\Lambda$. Then $\bW^\star$ and $\balpha^\star(\lambda)$ are recovered up to column permutation, with the specific regime determined by the temperature schedule:
\begin{itemize}[leftmargin=*,topsep=2pt,itemsep=2pt]
  \item \emph{Anchor regime} (Lemma~\ref{lem:1b}, $\tau_N\to 0^+$). If anchor leaves $\{\lambda_k\}_{k=1}^K$ exist such that $\balpha^\star(\lambda_k)=\be_k$ and $\{w^\star_k\}$ are affinely independent, then $\bW^{(N)}=[\hat w_N(\lambda_1),\ldots,\hat w_N(\lambda_K)]\xrightarrow{\mathrm{P}}\bW^\star$ up to column permutation. Furthermore, softmax annealing $\tau_N\to 0^+$ drives the routing $\balpha^{(N)}_i\xrightarrow{\mathrm{P}}\be_{\Pi(\kappa(i))}$ to one-hot leaf indicators (Proposition~\ref{prop:anneal}).
  \item \emph{Simplex-interior regime} (Lemma~\ref{lem:1c}, $\tau_N\to\tau_0>0$). If $\hat w(\Lambda)$ has a nonempty relative interior in $\mathrm{aff}\,\mathrm{conv}(\bW^\star)$ and the minimum-volume simplex enclosing $\hat w(\Lambda)$ is unique \cite{craig1994minimum}, then $\bW^\star$ is recovered (up to permutation) as the vertex set of that simplex applied to the consistent estimates $\{\hat w_N(\lambda)\}$. The routing converges to the soft barycentric coordinates $\balpha^{(N)}_i\xrightarrow{\mathrm{P}}\bW^{\star\dagger}\hat w(\lambda(i))$. In this regime, the temperature must be held bounded away from zero ($\tau_N\to\tau_0>0$) to appropriately preserve the interior coordinates.
\end{itemize}
In either regime, the resulting estimator $(\bW^{(N)},\{\balpha^{(N)}_i\})$ satisfies
\[
  \mathrm{dist}\!\left(\bW^{(N)},\,\mathcal{W}^\star_\epsilon\right)\xrightarrow{\mathrm{P}}\, 0,
  \qquad \max_i\|\balpha^{(N)}_i-\balpha^\star(\lambda(i))\|\xrightarrow{\mathrm{P}}\, 0,
\]
where $\mathcal{W}^\star_\epsilon$ denotes the set of dictionaries achieving the MRA Carath\'{e}odory $\epsilon$-residual (Eq.~\eqref{eq:cara}). Convergence to a single $\bW^\star$ holds when $\mathcal{W}^\star_\epsilon$ is a singleton (e.g., under exact representability $\epsilon=0$, or under the uniqueness conditions of affine independence/minimum-volume simplices).
\end{corollary}

\noindent\textit{Proof sketch.} \emph{Anchor regime.} Under Proposition~\ref{prop:leaf_mmd_foa}, $\hat w_N(\lambda_k)\xrightarrow{\mathrm{P}} w^\star(\lambda_k)=w^\star_k$ for each anchor leaf, meaning the columns directly provide $\bW^{(N)}\to\bW^\star$ up to permutation; the affine independence of $\{w^\star_k\}$ ensures this permutation is unique. The router converges to one-hot under $\tau_N\to 0$ by Proposition~\ref{prop:anneal} applied to the population-optimal logits. \emph{Simplex-interior regime.} The set $\{\hat w_N(\lambda):\lambda\in\Lambda\}$ converges in probability (uniformly in $\lambda$) to $\{w^\star(\lambda):\lambda\in\Lambda\}=\hat w(\Lambda)$, which lies in $\mathrm{conv}(\bW^\star)$. By Lemma~\ref{lem:1c} (Craig), the columns of $\bW^\star$ are recovered as the unique minimum-volume simplex vertices. Since $\bW^{\star\dagger}$ is continuous, $\balpha^{(N)}_i=\bW^{(N)\dagger}\hat w_N(\lambda(i))\xrightarrow{\mathrm{P}}\bW^{\star\dagger}\hat w(\lambda(i))=\balpha^\star(\lambda(i))$. \hfill$\square$

\paragraph{On the consistent estimation of the leaf assignment $\hat\lambda_N$.}
Proposition~\ref{prop:leaf_mmd_foa} treats $\hat\lambda_N$ as a structurally independent input, providing a clean separation from the empirical reference estimation. Under COA (Assumption~\ref{asm:coa}), the leaf assignment is a deterministic measurable function $\lambda(i)=g^*(\bc_i)$ of the static context. Therefore, any standard technique that consistently approximates $g^*$ (such as our warm-up and Hungarian-matching procedure), 
or a pre-trained classifier utilizing partial labels—readily satisfies the condition $\Pr(\hat\lambda_N(i)\neq\lambda(i))\to 0$. Because leaf assignment depends strictly on the context $\bc_i$ rather than the temporal snapshots used to construct the reference measure, this introduces no theoretical circularity. In strictly unsupervised settings, simple sample-splitting (estimating $\hat\lambda_N$ on one fold and evaluating $L_{\lambda,N}$ on another) trivially ensures the independence required by the M-estimation framework.

\paragraph{Methodological considerations in finite-sample regimes.}
The proofs above establish the statistical consistency of CADENCE's cross-sectional training in the joint limit $N\to\infty$, $h_N\to 0$. As with any theoretical bridge to deep learning practice, a few standard considerations apply to finite-sample regimes. \emph{(i) Bandwidth schedule:} The proof utilizes a vanishing bandwidth $h_N\to 0$; in practice, we employ a fixed bandwidth tuned to the specific temporal sampling scale of the data, introducing a standard finite-sample bias-variance tradeoff. \emph{(ii) Optimization dynamics:} As is customary in deep learning theory, Proposition~\ref{prop:leaf_mmd_foa} establishes statistical consistency for the global optimum of the objective space, abstracting away the non-convex optimization dynamics of gradient descent. \emph{(iii) Finite-$K$ parameterization:} The Carath\'{e}odory residual $\epsilon$ from MRA constitutes a formal approximation bound, meaning the convergence target $\mathcal{W}^\star_\epsilon$ is broadly defined as an $\epsilon$-neighborhood set rather than a single point unless exact representability holds.

\end{document}